\documentclass{article}

\usepackage[utf8]{inputenc}  
\usepackage[T1]{fontenc}
\usepackage{microtype}      % Improves character protrusion and expansion
\usepackage{graphicx}
\usepackage{booktabs}       % For professional-quality tables
\usepackage{amsmath,amssymb,amsthm}
\usepackage{mathtools}
\usepackage{lmodern}
\usepackage{xspace}
\usepackage{float}
\usepackage{caption}

\usepackage[dvipsnames,table]{xcolor} 
\usepackage{multirow}
\usepackage{makecell}
\usepackage{xltabular}
\usepackage{rotating}

\usepackage{hyperref}
\usepackage[capitalise]{cleveref}

\usepackage[accepted]{icml2026}   % Use for final camera-ready version

\usepackage{algorithm}
\usepackage{algorithmic}
\usepackage{listings}
\usepackage{listingsutf8}

\usepackage{subcaption} 
\usepackage{svg}
\usepackage{tikz}
\usepackage{overpic}
\usepackage{pgf} 
\usepackage{wrapfig}
\usepackage{varwidth}
\usepackage[most]{tcolorbox}
\tcbuselibrary{listings,breakable,skins}
\usepackage{subfiles}
\usepackage{verbatim}

\usepackage{siunitx} 
\usepackage{enumitem}
\usepackage{soul}
\sethlcolor{yellow}
\usepackage[textsize=tiny]{todonotes}

\soulregister\Cref7
\soulregister\cref7

\newcommand{\ie}{\textit{i.e.,}\xspace}
\newcommand{\eg}{\textit{e.g.,}\xspace}

\DeclareMathOperator{\SVID}{SVID}

\DeclareMathOperator{\sign}{sign}
\renewcommand{\vec}[1]{\mathbf{#1}}
\newcommand{\mat}[1]{\mathbf{#1}}
\newcommand{\g}{\vec{g}}
\newcommand{\h}{\vec{h}}
\newcommand{\x}{\vec{x}}
\newcommand{\y}{\vec{y}}

\newcommand{\R}{\mat{R}}
\newcommand{\B}{\mat{B}}
\newcommand{\W}{\mat{W}}
\newcommand{\s}{\vec{s}}

\newcommand{\THROUGHPUT}{291.88\xspace}
\newcommand{\SPEEDUP}{4.49$\times$\xspace}

\newcommand{\cmark}{\checkmark}
\newcommand{\xmark}{\phantom{\checkmark}} % Empty box-sized space

\newif\ifshowfixme \showfixmefalse
\ifshowfixme

  \newenvironment{fixmeblock}{\begingroup\color{red}}{\endgroup}  
\else

\fi

\newif\ifshowrevised \showrevisedtrue   
\ifshowrevised
  \newenvironment{revisedblock}{\begingroup\color{blue}}{\endgroup}  
\else  
    
\fi

\lstdefinelanguage{CUDA}[]{C++}{
  morekeywords={__global__,__device__,__host__,__forceinline__,__restrict__,__launch_bounds__,half,half2,dim3,uint2,uint4},
  sensitive=true
}
\lstdefinestyle{cudaStyle}{
  language=CUDA,                                     
  basicstyle=\ttfamily\scriptsize,      
  keywordstyle=\bfseries\color{blue!60!black},
  commentstyle=\itshape\color{gray!65!black},
  numbers=left, numberstyle=\tiny, numbersep=6pt,
  breaklines=true, breakatwhitespace=true,
  columns=fullflexible, keepspaces=true, showstringspaces=false,
  frame=single, rulecolor=\color{black!20},
  xleftmargin=10pt, framexleftmargin=8pt
}

\newtheorem{proposition}{Proposition}

\newenvironment{analysis}{\begin{proof}[\itshape Analysis.]}{\end{proof}}

\theoremstyle{plain}

\theoremstyle{definition}

\theoremstyle{remark}

\icmltitlerunning{RaBiT: Residual-Aware Binarization Training for Accurate and Efficient LLMs}

\begin{document}

\twocolumn[
  \icmltitle{\texorpdfstring{\raisebox{-0.1em}{\includegraphics[height=3.0ex]{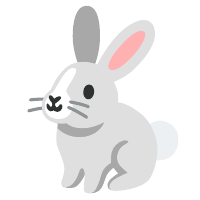}}~}{}%
  RaBiT: Residual-Aware Binarization Training\\for Accurate and Efficient LLMs}

  \icmlsetsymbol{equal}{*}
  \icmlsetsymbol{corr}{\dag}

  \begin{icmlauthorlist}
    \icmlauthor{Youngcheon You}{equal,sr}
    \icmlauthor{Banseok Lee}{equal,sr}
    \icmlauthor{Minseop Choi}{sr}
    \icmlauthor{Seonyoung Kim}{sr}

    \icmlauthor{Hyochan Chong}{sr}
    \icmlauthor{Changdong Kim}{sr}
    \icmlauthor{Youngmin Kim}{corr,sr}
    \icmlauthor{Dongkyu Kim}{corr,sr}
  \end{icmlauthorlist}

  \icmlaffiliation{sr}{Samsung Research, Seoul, Korea}

  \icmlcorrespondingauthor{Dongkyu Kim}{dongkyu.k@samsung.com}
  \icmlcorrespondingauthor{Youngmin Kim}{ym1012.kim@samsung.com}
  
  \icmlkeywords{Machine Learning, ICML}

  \vskip 0.3in
]

\printAffiliationsAndNotice{\icmlEqualContribution}

\begin{abstract}
    Efficient deployment of large language models (LLMs) requires extreme quantization, forcing a critical trade-off between low-bit efficiency and performance. Residual binarization enables hardware-friendly, matmul-free inference by stacking binary ($\pm$1) layers, but is plagued by pathological feature co-adaptation. We identify a key failure mode, which we term \textbf{inter-path adaptation}: during quantization-aware training (QAT), parallel residual binary paths learn redundant features, degrading the error-compensation structure and limiting the expressive capacity of the model. While prior work relies on heuristic workarounds (\eg path freezing) that constrain the solution space, we propose \textbf{RaBiT}, a novel quantization framework that resolves co-adaptation by algorithmically enforcing a residual hierarchy. Its core mechanism sequentially derives each binary path from a single shared full-precision weight, which ensures that every path corrects the error of the preceding one. This process is stabilized by a robust initialization that prioritizes functional preservation over mere weight approximation. RaBiT redefines the 2-bit accuracy-efficiency frontier: it achieves state-of-the-art performance, rivals even hardware-intensive Vector Quantization (VQ) methods, and delivers a \textbf{\SPEEDUP inference speed-up} over full-precision models on an RTX 4090. Code is available at {\hypersetup{urlcolor=black}\href{https://github.com/SamsungLabs/RaBiT}{\texttt{github.com/SamsungLabs/RaBiT}}}
\end{abstract}

% ===================== SECTION 1: INTRODUCTION =====================
\section{Introduction}
Model compression is essential for efficient deployment of large language models (LLMs). While 4-bit quantization methods~\citep{frantar2022gptq, lin2024awq} have emerged as a successful industry standard~\citep{vllm2023, sglang2024}, the relentless pursuit of greater efficiency is pushing the research frontier towards the 2-bit regime. This transition to lower bit compression introduces architectural trade-offs that characterize the current domain.

At the 2-bit frontier, two primary strategies present a choice between representational fidelity and hardware efficiency. Vector Quantization (VQ) methods typically yield high accuracy but incur hardware overhead due to lookup tables or complex rotations~\citep{tseng2024quip,tseng2024qtip, egiazarian2024extreme}. Conversely, residual binarization--stacking multiple binary layers--facilitates exceptional, matmul-free efficiency. However, this approach has struggled to preserve performance, due to fundamental training instabilities that prevent full potential realization~\citep{qbb2024,bitstack2024,mbok2025}.

The core promise of a residual architecture--that subsequent paths compensate for the errors of preceding ones--is fundamentally undermined by feature co-adaptation~\citep{hinton2012improving}, a pathological training dynamic where parallel components learn redundant features. In residual binarization, we identify a critical manifestation of this phenomenon, which we term \textbf{inter-path adaptation}. During standard quantization-aware training (QAT)~\citep{bengio2013estimating,hubara2018quantized}, a structurally agnostic global gradient is simultaneously applied to all paths. This forces each path to learn redundant features in a race to minimize the global objective, overriding their intended compensatory roles. The result is a breakdown of the residual hierarchy that severely limits the model's expressive power.

\begin{figure*}[t]
    \centering
    \begin{overpic}[width=0.865\linewidth, grid=false]{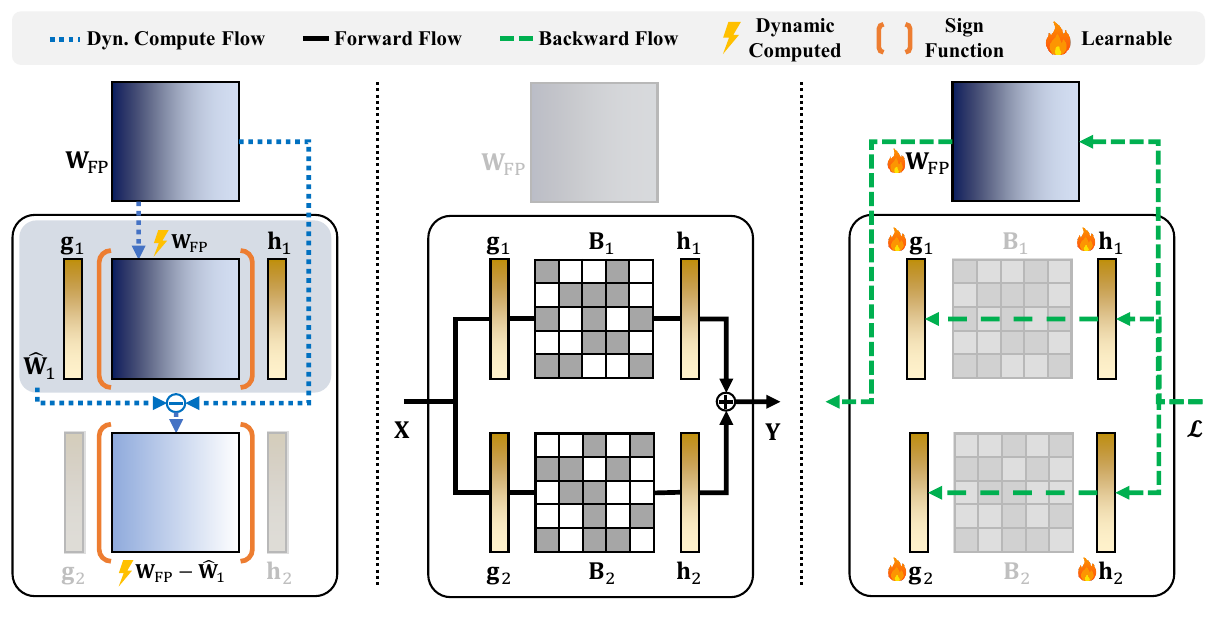}
        \put(13, 3){\makebox(0,0)[tl]{\footnotesize\textbf{(a)}}}
        \put(47, 3){\makebox(0,0)[tl]{\footnotesize\textbf{(b)}}}
        \put(82, 3){\makebox(0,0)[tl]{\footnotesize\textbf{(c)}}}
    \end{overpic}
    \caption{
        \textbf{Overview of the RaBiT Training Framework.}
        \textbf{(a) Dynamic Compute Process:} During training, binary paths are dynamically derived from a shared weight $\W_{\mathrm{FP}}$ to enforce a residual hierarchy.
        \textbf{(b) Forward Pass:} For inference, these paths execute in parallel for matmul-free efficiency.
        \textbf{(c) Backward Pass:} Gradients from the loss $\mathcal{L}$ update both the learnable scales ($\g_i, \h_i$) and the shared $\W_{\mathrm{FP}}$.
    }
    \label{fig:rabit_overview}
\end{figure*}

Previous efforts have employed heuristic constraints, such as path freezing~\citep{qbb2024, mbok2025}, which inherently limit the model's capacity to derive an optimal joint solution.
To address this, we propose \textbf{Residual-Aware Binarization Training (RaBiT)}, a QAT framework that resolves inter-path adaptation by design, as depicted in \Cref{fig:rabit_overview}. Instead of using independent latent weights, RaBiT maintains a single shared full-precision weight from which binary paths are sequentially derived on-the-fly, guided by learnable scales. This algorithmically enforces a residual hierarchy, training each path to correct its predecessor's error. Combined with a robust, function-aware initialization, RaBiT achieves state-of-the-art accuracy while delivering up to a \SPEEDUP inference speed-up and halving the training memory footprint.

Our contributions can be summarized as follows:
\begin{itemize}[itemsep=2pt, topsep=0pt, parsep=0pt, partopsep=0pt, leftmargin=*]
    \item We identify and analyze inter-path adaptation, a critical manifestation of feature co-adaptation in residual binarization, where the intended error-compensation structure breaks down during Standard QAT, as parallel paths become functionally redundant.
    \item We propose RaBiT, a novel QAT framework that resolves inter-path adaptation by enforcing residual coupling on-the-fly. The mechanism inherently \textit{halves the training memory footprint} and is stabilized by a robust function-aware initialization strategy to directly address the unstable dynamics of extreme QAT.
    \item We demonstrate that RaBiT achieves state-of-the-art accuracy at 2-bit precision, delivering up to a \SPEEDUP inference speed-up while maintaining competitive performance against hardware-intensive VQ methods through matmul-free operations.
\end{itemize}
    
\begin{table*}[t]
    \centering
    \caption{
        \textbf{Detailed Decomposition of MSE Loss across Representative Layers of Llama2-7B.}
        The table decomposes the total MSE into its core components. Path correlation ($\operatorname{Corr}(y_1, y_2)$) captures cancellation between binary paths, while residual alignment ($\operatorname{Corr}(R_1, y_2)$) captures how strongly the second path aligns the remaining functional residual.
    }
    \vspace{-5pt}
    \label{tab:detailed_mse_decomposition}
    \small
    \resizebox{0.95\textwidth}{!}{%
    \begin{tabular}{ll r|rrr r|c}
        \toprule
        \multirow{2}{*}{\textbf{Layer}} & \multirow{2}{*}{\textbf{Method}} 
        & \multicolumn{1}{c|}{\textbf{Total MSE}} 
        & \multicolumn{1}{c}{\textbf{Base Error}} 
        & \multicolumn{1}{c}{\textbf{Path Amp.}} 
        & \multicolumn{1}{c}{\textbf{Path Corr.}} 
        & \multicolumn{1}{c|}{\textbf{Covariance}} 
        & \multicolumn{1}{c}{\textbf{Residual Align.}} \\
        & 
        & \multicolumn{1}{c|}{($C'$ + Cov.) $\downarrow$} 
        & \multicolumn{1}{c}{\textbf{($C'$)}} 
        & \multicolumn{1}{c}{\textbf{($2\sigma_1\sigma_2$)}} 
        & \multicolumn{1}{c}{\textbf{($\operatorname{Corr}(y_1, y_2)$)} $\downarrow$} 
        & \multicolumn{1}{c|}{(Amp. $\times$ Corr.)} 
        & \multicolumn{1}{c}{\textbf{($\operatorname{Corr}(R_1, y_2)$) $\uparrow$}} \\
        \midrule
        \multirow{3}{*}{Layer 5 (Early)} 
        & Standard QAT & 0.0017 & 0.0019 & 0.0030 & -0.0752 & -0.0002 & 0.4395 \\
        & DB-LLM & 0.0027 & 0.0038 & 0.0021 & -0.4941 & -0.0010 & 0.2617 \\
        & MBOK & 0.0012 & 0.0013 & 0.0022 & -0.0320 & -0.0001 & 0.4805 \\
        & \textbf{RaBiT (Ours)} & \textbf{0.0009} & 0.0023 & 0.0028 & \textbf{-0.4961} & \textbf{-0.0014} & \textbf{0.6484} \\
        \cmidrule{2-8}
        \multirow{3}{*}{Layer 15 (Mid)} 
        & Standard QAT & 0.0156 & 0.0182 & 0.0213 & -0.1240 & -0.0026 & 0.4082 \\
        & DB-LLM & 0.0210 & 0.0236 & 0.0057 & \textbf{-0.4570} & -0.0026 & 0.2500 \\
        & MBOK & 0.0119 & 0.0139 & 0.0168 & -0.1187 & -0.0020 & 0.4570 \\
        & \textbf{RaBiT (Ours)} & \textbf{0.0094} & 0.0163 & 0.0200 & -0.3418 & \textbf{-0.0068} & \textbf{0.5820} \\
        \cmidrule{2-8}
        \multirow{3}{*}{Layer 25 (Late)} 
        & Standard QAT & 0.0482 & 0.0575 & 0.0725 & -0.1279 & -0.0093 & 0.4668 \\
        & DB-LLM & 0.0818 & 0.0878 & 0.0125 & \textbf{-0.4824} & -0.0060 & 0.2490 \\
        & MBOK & 0.0399 & 0.0474 & 0.0625 & -0.1196 & -0.0075 & 0.5039 \\
        & \textbf{RaBiT (Ours)} & \textbf{0.0327} & 0.0609 & 0.0798 & -0.3535 & \textbf{-0.0282} & \textbf{0.6172} \\
        \bottomrule
    \end{tabular}%
    }
\end{table*}

% ===================== SECTION 2: RELATED WORKS =====================
\section{Related Works}
\label{sec:related_works}

\paragraph{The Shift to QAT in Extreme Quantization.}
Post-Training Quantization (PTQ) methods, such as GPTQ~\citep{frantar2022gptq} and AWQ~\citep{lin2024awq}, have proven to be highly successful for compressing large language models to 3- or 4-bit precision by focusing on weight approximation.
However, these methods encounter substantial performance degradation at lower bit-widths (\eg 2-bit)~\citep{bitnet2023, billm2024, arbllm2025}, as the information loss inherent to coarse quantization cannot be mitigated solely by weight approximation. Consequently, more recent works target preserving global model functionality~\citep{paretoq2025} via quantization-aware training (QAT)~\citep{hubara2018quantized, krishnamoorthi2018quantizing}. QAT integrates low-precision arithmetic simulations into the fine-tuning process, enabling parameter adaptation to the target bit-width constraints. Although the non-differentiable nature of quantization presents challenges--typically managed via the Straight-Through Estimator (STE)~\citep{bengio2013estimating}--modern frameworks for binary models have achieved stability by updating a latent full-precision weight via surrogate gradients~\citep{bitnet2023,onebit2024,binarymos2024,littlebit2025}. This work builds upon this methodology to address the specific complexities of residual binary architectures.

\paragraph{Co-adaptation in Residual Binary Architectures.}
To enhance the limited expressive capacity of single low-bit layers, residual binarization stacks multiple low-bit paths ($\W \approx \sum_i \hat{\W}_{i}$), thereby achieving higher effective precision while retaining matmul-free efficiency~\citep{bitstack2024}. Nevertheless, this parallel structure is prone to \textbf{feature co-adaptation}~\citep{hinton2012improving}, where components learn redundant representations. This phenomenon previously motivated regularization techniques, such as Dropout~\citep{srivastava2014dropout}.

We identify a critical form of feature co-adaptation in residual binarization, termed inter-path adaptation, wherein a shared QAT gradient compels parallel paths to acquire redundant features, compromising the error-compensation hierarchy. Unlike prior works that depend on suboptimal heuristics like path freezing~\citep{qbb2024, mbok2025}, which restrict the solution space, RaBiT resolves this issue structurally, facilitating true joint optimization while algorithmically enforcing the hierarchy.

% ===================== SECTION 3: MOTIVATION =====================
\section{Motivation}
\label{sec:motivation}
The goal of quantization-aware training (QAT) is to make a quantized student model, $\mathbf{Y}_s$, functionally mimic its full-precision teacher, $\mathbf{Y}_t$. This is typically achieved by optimizing an objective that combines the final task loss with an intermediate knowledge distillation loss, often formulated as the mean squared error (MSE)~\citep{distilling2015,llmqat2023}. While our full training objective also includes the final KL divergence-based task loss, we initially focus our analysis on the MSE component for its analytical tractability. We formally prove that this error-decomposition logic and the optimality of RaBiT's residual coupling extend rigorously to the KL divergence objective (see \Cref{appendix:kl_analysis}). 

The additive structure of the MSE provides a clear window into how parallel paths interact. In a 2-bit residual architecture, the MSE between the teacher output $y_t$ and the student output $y_s=y_1+y_2$ can be decomposed. Using the Pearson correlation coefficient,\footnote{The relationship $\mathbb{E}[y_1y_2] \approx \sigma_1\sigma_2\operatorname{Corr}(y_1,y_2)$ relies on a zero-mean assumption for the path outputs. We empirically verify this, finding the omitted $\mathbb{E}[y_1]\mathbb{E}[y_2]$ term is less than 1\% of the covariance term and thus negligible.} this decomposition is:

\begin{equation}
\begin{aligned}
\label{eq:mse_decomp}
\operatorname{MSE}(y_t, y_s) &= \underbrace{(\mathbb{E}[y_t^2] + \mathbb{E}[y_1^2] + \mathbb{E}[y_2^2] - 2\mathbb{E}[y_t y_s])}_{C'} \\
&\quad + \underbrace{2\sigma_1 \sigma_2}_{\text{Path Amp.}} \cdot \underbrace{\operatorname{Corr}(y_1, y_2)}_{\text{Path Corr.}},
\end{aligned}
\end{equation}
where $C'$ represents the sum of correlation-independent error terms. This reveals a core principle: to minimize the MSE, the paths must be strongly \textit{negatively correlated}. A negative correlation transforms the interaction term into a substantial bonus that actively reduces the total loss, signifying effective error-cancellation.

To rigorously diagnose this true error-correction dynamic, we can re-examine the local MSE objective with respect to the residual error of the first path, $R_1 = y_t - y_1$, and the second path output, $y_2$. The objective $\mathbb{E}[(R_1 - y_2)^2]$ expands to:
\begin{equation}
\label{eq:residual_alignment}
\operatorname{MSE} \approx \sigma_{R_1}^2 + \sigma_{y_2}^2 - 2 \sigma_{R_1} \sigma_{y_2} \operatorname{Corr}(R_1, y_2)
\end{equation}
This formulation demonstrates that minimizing the MSE requires the second path to actively align with the residual error (\ie $\operatorname{Corr}(R_1, y_2) > 0$). A simple negative path-to-path correlation ($\operatorname{Corr}(y_1, y_2) < 0$) alone is insufficient, as it can be mechanically induced without capturing the true functional error.

To provide a concrete analysis, \Cref{tab:detailed_mse_decomposition} decomposes these MSE metrics across representative layers of Llama2-7B. The results are definitive. Standard QAT, which optimizes independent latent weights via a shared global gradient, suffers from \textbf{inter-path adaptation} (see \Cref{appendix:proofs}); it yields a path correlation close to zero, failing to meaningfully reduce the total error. Conversely, heuristic splitting methods like DB-LLM~\citep{chen2024dbllm} mechanically force a negative path correlation, but fundamentally fail to maximize residual alignment (\eg 0.26 in Layer 5), indicating that their paths do not actively track the actual functional error. MBOK~\citep{mbok2025} partially improves residual alignment over Standard QAT, but its path-freezing heuristic yields only weak path anti-correlation and limited loss-reducing covariance.

In contrast, RaBiT explicitly derives the second path from the dynamic residual on-the-fly. This structurally guarantees that the second path actively cancels the functional error, achieving the highest residual alignment (\eg 0.65 in Layer 5). This true error-correction transforms the interaction term into a significant loss-reducing bonus, systematically lowering the total MSE. This principled enforcement creates a more stable optimization landscape, leading to better generalization and superior performance.

% ===================== SECTION 4: METHOD =====================

\section{Method}
\label{sec:method}
We introduce RaBiT, a novel QAT framework that prevents interference between the parallel paths of stacked binary architectures. To achieve this, RaBiT enforces a clear error-correction role for each path using a novel \textbf{coupled training} loop and stabilizes the process with a \textbf{function-aware initialization} strategy. An overview of the RaBiT training framework is illustrated in \Cref{fig:rabit_overview}.

\subsection{The Residual Binarization Architecture}
\label{subsec:architecture}
To efficiently achieve low-bit precision (\eg 2-bit), we adopt a residual architecture built upon highly efficient binary building blocks.

\paragraph{Binary Building Blocks.}
The fundamental component is the dual-scale binarization framework. We define the approximation of a weight matrix $\hat{\W}$ using a notation that highlights the underlying element-wise scaling operations:
\begin{equation}
    \hat{\W} = \g \odot \B \odot \h.
\end{equation}
Here, $(\hat{\W})_{ij}$ is computed as $g_i B_{ij} h_j$. $\B \in \{-1,+1\}^{d_{\mathrm{out}}\times d_{\mathrm{in}}}$ represents the binary core matrix, while $\g \in \mathbb{R}^{d_{\mathrm{out}}}$ and $\h \in \mathbb{R}^{d_{\mathrm{in}}}$ are full-precision, per-channel scaling vectors. This formulation facilitates matmul-free efficiency. For an input vector $\x \in \mathbb{R}^{d_{\mathrm{in}}}$, the output $\y \in \mathbb{R}^{d_{\mathrm{out}}}$ is computed as $\y = \g \odot (\B(\h \odot \x))$, implementable via additions and subtractions without costly matrix multiplications.

\paragraph{Multi-bit Approximation via Stacking.}
To increase representational capacity (\eg to 2-bit) while maintaining efficiency, $k=2$ binary paths are stacked in parallel. The effective weight is the sum of binarized terms:
\begin{equation}
\label{eq:stacked-binary-weight}
\hat{\W}^{(k)} = \sum_{i=1}^{k} \hat{\W}_i = \sum_{i=1}^{k} \g_i \odot \B_i \odot \h_i.
\end{equation}
This architecture preserves matmul-free execution, as the forward pass accumulates outputs from each path. Note that while this defines the core bit-width (\eg 2-bit for $k=2$), the effective bit-width also includes the storage cost of the low-dimensional auxiliary scaling parameters $\g_i$ and $\h_i$ (\eg a nominal 2-bit model corresponds to $\sim$2.02 effective bits).

\subsection{Coupled Training for Co-Adaptation Mitigation}
\label{subsec:coupled_training}
To prevent inter-path co-adaptation, RaBiT diverges from standard methods that train independent latent weights for each binary path. Instead, a \textbf{single shared full-precision (FP) weight} $\W_{\mathrm{FP}}$ is used as an anchor for the entire residual structure.

\paragraph{The Coupled Forward Pass.} The mechanism relies on a dynamic forward pass. In a 2-bit configuration ($k=2$), binary core matrices $\B_1$ and $\B_2$ are not explicitly stored. Rather, they are re-computed during each pass from the shared weight $\W_{\mathrm{FP}}$ (\Cref{fig:rabit_overview}a). This procedure algorithmically enforces an error-compensation hierarchy. While binary cores are dynamically derived, scaling vectors $\{\g_i, \h_i\}$ remain as independent learnable parameters to capture path-specific magnitudes. The forward pass consists of 3 steps:
\begin{enumerate}[itemsep=1.5pt, topsep=0pt, leftmargin=*]
    \item \textbf{Path 1 Derivation}: The first binary core, $\B_1$, is determined by directly binarizing the shared weight: $\B_1 = \sign(\W_{\mathrm{FP}})$. This core is combined with corresponding learnable scaling vectors, $\g_1$ and $\h_1$, to reconstruct the first-path approximation $\hat{\W}_1 = \g_1 \odot \B_1 \odot \h_1$.
    \item \textbf{Residual Calculation}: The residual error, $\R_{1}$, is computed by subtracting the \textit{reconstructed} first path from the shared weight: $\R_{1} = \W_{\mathrm{FP}} - \hat{\W}_1$.
    \item \textbf{Path 2 Derivation}: The second binary core, $\B_2$, is determined by binarizing this calculated residual error: $\B_2 = \sign(\R_1)$. The final effective weight for the forward pass is the sum of the two reconstructed paths: $\hat{\W}^{(2)} = \hat{\W}_1 + (\g_2 \odot \B_2 \odot \h_2)$.
\end{enumerate}
A critical distinction is made between dynamically deriving binary cores $\B_i := \sign(\R_{i-1})$ and treating scaling vectors $\{\g_i, \h_i\}$ as independent parameters. This separation ensures computational feasibility and training stability. Recalculating optimal scales at every step--\eg via Singular Value Decomposition (SVD)--would impose prohibitive computational costs. By maintaining scales as learnable parameters, the optimizer applies state accumulation (\eg momentum) to fine-tune initialized values effectively (\cref{subsec:stable_init}). Such data-adaptive tuning enables the error-compensation hierarchy to learn optimal magnitudes for each path.

\paragraph{Backward Pass and Parameter Updates.} The backward pass facilitates stable parameter updates. Gradients from the loss $\mathcal{L}$ propagate to both the independent scaling vectors $\{\g_i,\h_i\}$ and the shared weight $\W_{\mathrm{FP}}$, as depicted in \Cref{fig:rabit_overview}c.
\begin{itemize}[itemsep=1.5pt, topsep=0pt, leftmargin=*]
    \item \textbf{Gradient for the Shared Weight}:
          To update the single shared weight $\mathbf{W}_{\mathrm{FP}}$, RaBiT employs an \textbf{effective-weight gradient}.
          This functions as a Straight-Through Estimator (STE) for the \textit{entire coupled derivation process}.
          The gradient is computed with respect to the final effective weight $\hat{\mathbf{W}}^{(k)} = \sum_i \hat{\mathbf{W}}_i$ and is directly passed to update $\mathbf{W}_{\mathrm{FP}}$. To ensure dimensional consistency for general cases ($d_{\mathrm{in}} \neq d_{\mathrm{out}}$), we formulate the update as:
          \begin{equation}
          \label{eq:RaBiT-grad-wfp}
          \nabla_{\mathbf{W}_{\mathrm{FP}}}
          \;\approx\;
          \nabla_{\hat{\mathbf{W}}^{(k)}}\mathcal{L}
          = \frac{\partial\mathcal{L}}{\partial \mathbf{Y}} \mathbf{X}^{\top}.
          \end{equation}
          In this context, $\mathcal{L}$ denotes the task loss, while $\mathbf{X} \in \mathbb{R}^{d_{\mathrm{in}} \times N}$ and $\mathbf{Y} \in \mathbb{R}^{d_{\mathrm{out}} \times N}$ represent the full input and output matrices for the mini-batch. By recomputing binary paths from the updated $\W_{\mathrm{FP}}$ at every step, the system forces each path to correct the most recent residual error, thereby preserving the error-compensation hierarchy.
    \item \textbf{Gradient for Learnable Scales}:
          The scaling vectors $\{\g_i,\h_i\}$ function as standard learnable parameters and receive gradients via the chain rule. For a mini-batch of size $N$, gradients are accumulated over each sample while treating the dynamic binary cores ($\B_i$) as constants:
            \begin{equation}
            \label{eq:RaBiT-grad-gh}
            \begin{split}
                \nabla_{\g_i} &= \sum_{n=1}^{N} \Delta_{n} \odot \bigl(\B_i (\h_i \odot \mathbf{x}_{n})\bigr), \\
                \nabla_{\h_i} &= \sum_{n=1}^{N} \bigl(\B_i^{\!\top}(\Delta_{n} \odot \g_i)\bigr) \odot \mathbf{x}_{n}.
            \end{split}
            \end{equation}
          Here, $\mathbf{x}_n$ denotes the input vector corresponding to the $n$-th column of $\mathbf{X}$, and $\Delta_{n} = (\partial\mathcal{L}/\partial \mathbf{y}_n)$ represents the upstream gradient from the layer's output vector $\mathbf{y}_n$.
\end{itemize}

During inference, the final binary cores $\{\B_i\}$ are derived from the trained $\W_{\mathrm{FP}}$ and frozen, allowing $\W_{\mathrm{FP}}$ to be discarded. This yields a highly efficient architecture where independent paths execute in a fully parallel, matmul-free manner. Furthermore, the single-weight design reduces memory usage for optimizer states by 50\%, addressing a primary bottleneck in LLM fine-tuning. The complete training procedure is detailed in \Cref{alg:training}.

\subsection{Stable Initialization for Functional Preservation}
\label{subsec:stable_init}
Quantization-aware training (QAT) in the 2-bit regime is sensitive to initialization, often trapping models in suboptimal minima~\citep{nagel2020up, llmqat2023}.
To address this, we implement a two-stage initialization process designed to preserve model \textit{functionality} rather than simply approximating weight values.

\paragraph{1. Iterative Residual SVID.}
The primary objective is to identify binary paths that jointly approximate the target matrix. Standard greedy decomposition is suboptimal as early choices irreversibly bias subsequent paths. Therefore, we employ \textbf{Iterative Residual Sign-Value-Independent Decomposition (SVID)}, a Gauss-Seidel style iteration facilitating path co-adaptation. The process refines scales $\{\g_i, \h_i\}$ and binary cores $\{\B_i\}$ for each path $i=1, \dots, k$ over iterations $t=1, \dots, T$ as follows:
\begin{equation}
\label{eq:init-iter}
\left\{
\begin{aligned}
    \mathbf{R}_i^{(t)} &\coloneqq \mathbf{W}_{\mathrm{FP}} - \textstyle \sum_{j<i} \hat{\mathbf{W}}_j^{(t)} \\
    &\quad \quad \quad - \textstyle \sum_{j>i} \hat{\mathbf{W}}_j^{(t-1)}, \\
    \mathbf{B}_i^{(t)}, \mathbf{g}_i^{(t)}, \mathbf{h}_i^{(t)} &\coloneqq \operatorname{SVID}(\mathbf{R}_i^{(t)}), \\
    \hat{\mathbf{W}}_i^{(t)} &\coloneqq \mathbf{g}_i^{(t)} \odot \mathbf{B}_i^{(t)} \odot \mathbf{h}_i^{(t)}.
\end{aligned}
\right.
\end{equation}
$\SVID(\cdot)$~\citep{onebit2024} extracts optimal per-channel scales via rank-1 SVD approximation on the magnitudes. In practice, this iterative process is applied to a preconditioned target matrix $\mathbf{W'}$ rather than the raw weights $\W_{\mathrm{FP}}$.

\paragraph{2. I/O Channel Importance-Scaled Preconditioning.}
To focus decomposition on functionally critical weight components, we precondition the full-precision weights before residual decomposition. 
This design is motivated by the local-sensitivity intuition underlying Fisher-based and K-FAC-style importance estimation~\citep{molchanov2019importance,martens2015optimizing}: weight errors on channels with large input activations or large output-gradient magnitudes have greater functional impact.
Drawing on recent methods for functional saliency preservation~\citep{dbf2025}, the raw weights $\W_{\mathrm{FP}}$ are therefore re-weighted to generate a target $\W'$. 
Using a calibration dataset, input activation magnitudes ($s_{\mathrm{in}}$) and output gradient magnitudes ($s_{\mathrm{out}}$) are computed as the maximum absolute values and normalized to re-weight the full-precision matrix:

\begin{equation}
\label{eq:init-calib}
\W' \;=\; \s_{\mathrm{out}}^{\alpha_{\mathrm{out}}} \odot \W_{\mathrm{FP}} \odot \s_{\mathrm{in}}^{\alpha_{\mathrm{in}}}.
\end{equation}

Following Iterative Residual SVID, the resulting scales are mapped back to the original domain: $\g_i = \s_{\mathrm{out}}^{-\alpha_{\mathrm{out}}} \odot \g_i'$ and $\h_i = \s_{\mathrm{in}}^{-\alpha_{\mathrm{in}}} \odot \h_i'$. This approach reduces initial task loss and stabilizes the QAT startup phase (see \Cref{alg:initialization}, \Cref{tab:initialization_analysis}, and \Cref{fig:kl_loss_svid_iter}).

\begin{table*}[t]
  \centering
  \caption{
    \textbf{Comparison with State-of-the-Art 2-3-bit Methods on Llama Models.}
    We report perplexity (PPL $\downarrow$) and zero-shot QA Average ($\uparrow$).
    For the 2-bit results, the best and runner-up are marked in \textbf{bold} and \underline{underlined}, respectively.
    RaBiT achieves state-of-the-art (SOTA) performance on Llama2-7B and Llama3-8B, while showing highly competitive results on Llama2-13B.
    }
  \setlength{\tabcolsep}{4.5pt}
  \renewcommand{\arraystretch}{1.12}
  \resizebox{0.90\textwidth}{!}{
      \begin{tabular}{lcccccccccccc}
\toprule
\multirow{2}{*}{\textbf{Methods}} 
  & \multicolumn{4}{c}{\textbf{Llama2-7B}} 
  & \multicolumn{4}{c}{\textbf{Llama2-13B}} 
  & \multicolumn{4}{c}{\textbf{Llama3-8B}} \\
\cmidrule(lr){2-5}\cmidrule(lr){6-9}\cmidrule(lr){10-13}
  & \textbf{Bit} & \textbf{Wiki2$\downarrow$} & \textbf{C4$\downarrow$} & \textbf{QA Avg$\uparrow$}
  & \textbf{Bit} & \textbf{Wiki2$\downarrow$} & \textbf{C4$\downarrow$} & \textbf{QA Avg$\uparrow$}
  & \textbf{Bit} & \textbf{Wiki2$\downarrow$} & \textbf{C4$\downarrow$} & \textbf{QA Avg$\uparrow$} \\
\midrule
Baseline
  & 16    & 5.12 & 6.63 & 62.26
  & 16    & 4.57 & 6.05 & 65.46 
  & 16    & 5.75 & 8.32 & 68.66 \\
GPTQ                    
  & 2.1 & 50.75 & 36.76 & 39.16 
  & 2.1 & 43.84 & 23.07 & 43.72 
  & 2  & 1.21e3 & 4.97e2   & 35.59 \\
EfficientQAT                   
  & 2.1  & 6.42   & 8.34   & 57.75
  & 2.1  & 5.58   & 7.40   & 62.07
  & 2.1  & 8.75   & 12.09   & 60.63 \\
\midrule
AQLM                    
  & 2.3  & 6.29 & 8.56 & 58.57 
  & 2.2  & 5.41 & 7.20 & 61.58 
  & 2.3   & 7.23 & 10.32 & 64.12\\
  % wiki 6.76(8192)
QuIP\#                  
  & 2  & 6.19 & 8.16 & 58.23 % & 2.02
  & 2  & 5.35 & 7.20 & 61.96 
  & 2  & 8.70 & 12.04 & 63.89 \\
QTIP                   
  & 2  & \underline{5.86}   & \underline{7.73}   & \underline{58.97}
  & 2  & \textbf{5.11}   & \textbf{6.85}   & \textbf{62.92}
  & 2  & \underline{7.52} & \underline{10.76} & \underline{63.88} \\
\midrule
\multirow{2}{*}{BitStack} 
  & 3 & 6.91 & 9.10 & 56.54
  & 3 & 5.90  & 7.86 & 61.06
  & 3 & 12.38 & 17.51 & 58.41  \\
  & 2 & 29.97  & 34.91 & 40.12 
  & 2 & 67.98  & 72.60 & 39.38
  & 2 & 2.75e3  & 1.93e3 & 36.21  \\
DB-LLM                  
  & 2  & 7.23 & 9.62 & 55.12 % 2.01 
  & 2  & 6.19 & 8.38 & 59.41 % 2.01 
  & 2 & 12.08 & 16.80 & 50.92 \\
\multirow{2}{*}{MBOK} 
  & 3 & 6.13  & 8.13  & 54.63 
  & 3 & 5.14  & 6.94  & 62.73 
  & 3 & 7.81  & 11.29  & 61.08  \\
  & 2 & 6.99 & 9.38  & 53.63 
  & 2 & 5.76  & 7.89  & 60.58 
  & 2 & 10.74  & 14.61  & 54.41  \\    
\multirow{2}{*}{DBF}
  & 2.3 & 5.81  & 7.69  & 59.84
  & 2.3 & 5.15  & 6.85  & 62.53
  & 2.3 & 7.22  & 10.34  & 64.84  \\
  & 2 & 6.10  & 8.05  & 58.42
  & 2 & 5.33  & 7.13  & 61.53
  & 2 & 7.78  & 10.99  & 62.90  \\
\midrule
\multirow{2}{*}{RaBiT (Ours)} 
  & 3 & 5.36  & 7.06  & 63.05 
  & 3 & 4.84  & 6.51  & 64.09
  & 3 & 6.58  & 9.54  & 65.61  \\
  & 2 & \textbf{5.78}  & \textbf{7.64}  & \textbf{61.51} 
  & 2 & \underline{5.15}  & \underline{6.95}  & \underline{62.10} 
  & 2 & \textbf{7.34}  & \textbf{10.52} & \textbf{64.13} \\
\bottomrule
\end{tabular}

      \label{tab:main_results}
    }
\end{table*}

% ===================== SECTION 5: EXPERIMENTS =====================

\section{Experiments}
\subsection{Experimental Settings}
\label{sec:exp_settings}
\paragraph{Setup.}
The RaBiT framework is evaluated using the Llama2, Llama3, and Gemma3 model families~\citep{llama2_2023,llama3_2024,gemma3}. For quantization-aware training (QAT), a calibration dataset of 200 million tokens is derived from the combined WikiText-2 and C4 datasets~\citep{binarymos2024}. Performance assessment relies primarily on perplexity (PPL) calculated on validation sets with a context length of 4096.

\paragraph{Evaluation Benchmarks.}
We report the average zero-shot accuracy (QA Avg.) across five standard reasoning benchmarks, including HellaSwag, PIQA, WinoGrande, ARC-e and ARC-c~\citep{hellaswag2019, piqa2020, winogrande2021, arc2018}. Detailed breakdowns for these standard benchmarks are provided in \Cref{appendix:ext_results}. Furthermore, to test the model's robustness on harder tasks involving complex reasoning and instruction following, we include evaluation on Big Bench Hard (BBH), GPQA, MMLU-Pro, and IFEval~\citep{suzgun2023bbh, rein2024gpqa, wang2024mmlupro, zhou2023ifeval}. These challenging benchmarks serve to highlight the functional preservation of the quantized models beyond basic language modeling metrics.

\begin{table*}[t]
  \centering
  \caption{\textbf{Comparison with State-of-the-Art 2-bit Methods on Gemma3 Models.} We report perplexity (PPL $\downarrow$) and zero-shot QA Average ($\uparrow$). The context length is 4096. RaBiT achieves state-of-the-art (SOTA) or highly competitive performance across the Gemma3 1B, 4B, and 12B suite.}
  \setlength{\tabcolsep}{4.5pt}
  \renewcommand{\arraystretch}{1.12}
  \resizebox{0.90\textwidth}{!}{
      \begin{tabular}{lcccccccccccc}
\toprule
\multirow{2}{*}{\textbf{Methods}} 
  & \multicolumn{4}{c}{\textbf{Gemma3-1B}} 
  & \multicolumn{4}{c}{\textbf{Gemma3-4B}} 
  & \multicolumn{4}{c}{\textbf{Gemma3-12B}} \\
\cmidrule(lr){2-5}\cmidrule(lr){6-9}\cmidrule(lr){10-13}
  & \textbf{Bit} & \textbf{Wiki2$\downarrow$} & \textbf{C4$\downarrow$} & \textbf{QA Avg$\uparrow$}
  & \textbf{Bit} & \textbf{Wiki2$\downarrow$} & \textbf{C4$\downarrow$} & \textbf{QA Avg$\uparrow$}
  & \textbf{Bit} & \textbf{Wiki2$\downarrow$} & \textbf{C4$\downarrow$} & \textbf{QA Avg$\uparrow$} \\
\midrule
Baseline 
  & 16    & 9.80 & 13.69 & 57.82 & 16    & 6.88 & 10.44 & 67.60 & 16    & 5.50 & 9.28 & 73.45 \\
DBF
  & 2 & 13.28  & 17.57  & 51.98
  & 2 & 8.72  & 12.71  & 60.91
  & 2 & 6.97  & 10.60  & 68.37  \\
QTIP                   
  & 2  & 13.14   & 17.36   & 50.30
  & 2  & 8.31   & 12.21   & \textbf{63.47}
  & 2  & \textbf{6.65} & 10.25 & \textbf{69.69} \\
\midrule
RaBiT (Ours)
  & 2 & \textbf{11.27} & \textbf{15.54}  & \textbf{53.18}
  & 2 & \textbf{8.09}  & \textbf{11.91}  & 62.21
  & 2 & 6.66  & \textbf{10.18}  & 68.85 \\
\bottomrule
\end{tabular}

      \label{tab:gemma3_results}
    }
\end{table*}

\begin{table*}[t]
  \centering
  \caption{\textbf{Comparison with State-of-the-Art 2-bit Methods on Qwen3 Models.} We report perplexity (PPL $\downarrow$) and zero-shot QA Average ($\uparrow$). RaBiT achieves state-of-the-art (SOTA) performance on Qwen3-1.7B and Qwen3-4B, while showing highly competitive results on Qwen3-8B.}
  \setlength{\tabcolsep}{4.5pt}
  \renewcommand{\arraystretch}{1.12}
  \resizebox{0.90\textwidth}{!}{
      \begin{tabular}{lcccccccccccc}
\toprule
\multirow{2}{*}{\textbf{Methods}} 
  & \multicolumn{4}{c}{\textbf{Qwen3-1.7B}} 
  & \multicolumn{4}{c}{\textbf{Qwen3-4B}} 
  & \multicolumn{4}{c}{\textbf{Qwen3-8B}} \\
\cmidrule(lr){2-5}\cmidrule(lr){6-9}\cmidrule(lr){10-13}
  & \textbf{Bit} & \textbf{Wiki2$\downarrow$} & \textbf{C4$\downarrow$} & \textbf{QA Avg$\uparrow$}
  & \textbf{Bit} & \textbf{Wiki2$\downarrow$} & \textbf{C4$\downarrow$} & \textbf{QA Avg$\uparrow$}
  & \textbf{Bit} & \textbf{Wiki2$\downarrow$} & \textbf{C4$\downarrow$} & \textbf{QA Avg$\uparrow$} \\
\midrule
Baseline 
  & 16 & 7.12 & 10.55 & 64.21 & 16 & 5.92 & 9.45 & 70.12 & 16 & 5.15 & 8.50 & 74.88 \\
DBF
  & 2 & 11.89 & 15.95 & 59.25
  & 2 & 9.65 & 13.55 & 66.01
  & 2 & 8.36 & 12.00 & 70.48 \\
QTIP                   
  & 2 & 11.19 & 15.37 & 56.82
  & 2 & 8.76 & 12.59 & 66.13
  & 2 & 7.59 & 11.13 & \textbf{70.93} \\
\midrule
RaBiT (Ours)
  & 2 & \textbf{10.19} & \textbf{14.49} & \textbf{60.38}
  & 2 & \textbf{8.27} & \textbf{12.08} & \textbf{66.66}
  & 2 & \textbf{7.47} & \textbf{11.00} & 70.25 \\
\bottomrule
\end{tabular}

      \label{tab:qwen3_results}
    }
\end{table*}

\paragraph{Training Details.}
The training process uses a QAT framework integrated with Knowledge Distillation (KD)~\citep{distilling2015,llmqat2023}, employing the full-precision model as the teacher. The objective function combines the Kullback–Leibler (KL) divergence loss on output logits with intermediate mean squared error (MSE) losses: $\mathcal{L}_\mathrm{total} = \mathcal{L}_\mathrm{kl} + \gamma \sum_{i} \mathcal{L}_{\mathrm{inter}, i}$. We set $\gamma=100$ for Llama models and $\gamma=0$ for Gemma3 to mitigate instability caused by large activation ranges~\citep{unsloth2025gemma3}. Models are trained for 6 epochs using the Muon optimizer~\citep{jordan2024muon} and our proposed function-aware initialization. We provide comprehensive hyperparameter configurations in \Cref{appendix:hyper_parameters}.

\paragraph{Baselines.}
We compare RaBiT against a broad spectrum of state-of-the-art 2- to 3-bit quantization methods. These include: (1) standard post-training quantization methods such as GPTQ and EfficientQAT~\citep{frantar2022gptq, chen2025efficientqat}; (2) high-accuracy Vector Quantization (VQ) approaches including AQLM, QuIP\#, and QTIP~\citep{egiazarian2024extreme, tseng2024quip, tseng2024qtip}, which typically incur high hardware overhead; and (3) hardware-efficient binary and residual methods such as BitStack, DB-LLM, MBOK, and DBF~\citep{bitstack2024, chen2024dbllm, mbok2025, dbf2025}\footnote{We rely on our re-implementation for DB-LLM and MBOK.}, which serve as the most direct architectural comparisons.

\subsection{Main Results}
The empirical results, summarized in \Cref{tab:main_results,tab:gemma3_results,tab:qwen3_results,tab:zeroshot_results}, demonstrate that RaBiT consistently redefines the state-of-the-art for 2-bit quantization, exhibiting superior performance across all tested model architectures and datasets.

\paragraph{Dominance over Hardware-Efficient Methods.}
RaBiT significantly outperforms existing matmul-free binary and residual methods, highlighting the efficacy of the proposed optimization strategy. On the Llama2-7B benchmark, RaBiT achieves a WikiText-2 perplexity (PPL) of 5.78. This represents a substantial improvement over direct competitors such as MBOK (6.99 PPL) and DBF (6.10 PPL). The performance gap becomes even more pronounced on larger models and more complex datasets, which underscores the severe performance penalty incurred by the inter-path adaptation issues that prior methods fail to address. Notably, we observed that methods lacking our principled coupled design, such as BitStack, suffer from catastrophic instability on newer architectures like Llama3-8B (degrading to 2.75e3 PPL), whereas RaBiT successfully maintains robust convergence and high fidelity.

\begin{table}[t!]
  \centering
    \caption{\textbf{Zero-Shot Evaluation on Challenging Benchmarks.}
    We report BBH, GPQA, MMLU-Pro, and IFEval results for Llama2-13B and Llama3-8B.}
  \label{tab:zeroshot_results}
  \resizebox{\columnwidth}{!}{%
    \renewcommand{\arraystretch}{1.35}
\begin{tabular}{l c c c @{\hspace{8pt}} c c c}

  \toprule
  & \multicolumn{3}{c}{\textbf{Llama2-13B}} & \multicolumn{3}{c}{\textbf{Llama3-8B}} \\
  \cmidrule(lr){2-4} \cmidrule(lr){5-7}
  \textbf{Methods} &
  \textbf{Baseline} & \textbf{QTIP} & \textbf{RaBiT} &
  \textbf{Baseline} & \textbf{QTIP} & \textbf{RaBiT} \\
  \midrule
  Bit & 16 & 2 & 2 & 16 & 2 & 2 \\
  \midrule
  BBH            & 40.99 & 33.36 & \textbf{37.72} & 45.84 & 36.27 & \textbf{36.78} \\
  GPQA           & 27.45 & 25.75 & \textbf{26.77} & 30.85 & 24.56 & \textbf{28.62} \\
  MMLU-Pro           & 24.81 & 16.69 & \textbf{19.44} & 14.91 & 19.24 & \textbf{19.65} \\
  IFEval         & 23.83 & \textbf{25.74} & 24.63 & 32.51 & \textbf{15.60} & 15.42 \\
  \midrule
  \textbf{Average}  & 29.27 & 25.38 & \textbf{27.14} & 31.03 & 23.92 & \textbf{25.12} \\
  \bottomrule
\end{tabular}
\renewcommand{\arraystretch}{1.0}
  }
  \vspace{-10pt}
\end{table}

\paragraph{Achieving VQ-Level Accuracy with Binary Efficiency.}
Importantly, RaBiT effectively narrows the accuracy gap with hardware-intensive Vector Quantization (VQ) methods, achieving comparable inference-time performance while preserving the efficient matmul-free architecture. On Llama2-7B, RaBiT's perplexity of 5.78 edges out the leading VQ method, QTIP (5.86 PPL). This trend holds for downstream reasoning tasks as well; RaBiT achieves a 61.51\% average zero-shot accuracy on Llama2-7B, surpassing QTIP's 58.97\% and demonstrating superior functional preservation. 

This robustness is further corroborated by results on Llama3-8B, where RaBiT maintains strong performance (7.34 PPL) while other VQ methods like QuIP\# suffer from severe degradation (8.70 PPL). As shown in \Cref{tab:gemma3_results} and \Cref{tab:qwen3_results}, this advantage generalizes beyond the Llama family. On the Gemma3 (1B/4B/12B) and Qwen3 (1.7B/4B/8B) suites, RaBiT consistently delivers competitive perplexity and robust zero-shot QA accuracy. For instance, on the instruction-tuned Qwen3-4B, RaBiT achieves a 66.66\% average zero-shot accuracy, outperforming both DBF (66.01\%) and QTIP (66.13\%), while simultaneously improving Wiki2 perplexity (8.27 vs. 8.76 for QTIP). Even on the larger Qwen3-8B model, RaBiT remains highly competitive with VQ methods, demonstrating the scalability and strong functional preservation of our approach across diverse architectures.

\paragraph{Performance on Harder Tasks.}
To further validate the model's capabilities in complex scenarios, we evaluated performance on challenging benchmarks including BBH, GPQA, MMLU-Pro, and IFEval. As presented in \Cref{tab:zeroshot_results}, RaBiT outperforms QTIP on average (27.14 vs. 25.38 on Llama2-13B) and retains substantially more capability relative to the full-precision baseline than prior quantization techniques. This indicates that RaBiT's coupled training strategy preserves the delicate internal representations required for advanced reasoning and instruction following, which are often lost in standard binary quantization.

\begin{table}[t!]
    \centering
    \caption{\textbf{Ablation on RaBiT} (Llama2-7B PPL). The analysis isolates the impact of Iterative Residual SVID (\textbf{I}) and I/O Channel Importance-Scaled Preconditioning (\textbf{S}).}
    \label{tab:ablation-RaBiT}
    \begin{small}
        \begin{center}
            \resizebox{0.8\columnwidth}{!}{
                \begin{tabular}{@{}ccccc@{}}
\toprule
\textbf{Training Method} & \textbf{I} & \textbf{S} & \textbf{WikiText-2 $\downarrow$} \\
\midrule
\multirow{4}{*}{\makecell[c]{Standard QAT}}
& \xmark & \xmark & 6.55 \\ % Baseline: Greedy SVID
& \cmark & \xmark & 6.21 \\
& \xmark & \cmark & 6.31 \\
& \cmark & \cmark & 6.18 \\ 
\midrule
\multirow{4}{*}{\makecell[c]{\textbf{Coupled QAT}\\\textbf{(RaBiT)}}}
& \xmark & \xmark & 5.84 \\
& \cmark & \xmark & 5.80 \\
& \xmark & \cmark & 5.81 \\
& \cmark & \cmark & \textbf{5.78} \\ 
\bottomrule
\end{tabular}
            }
        \end{center}
    \end{small}
    \vspace{-5pt}
\end{table}

\begin{table*}[ht]
\centering
\caption{
    \textbf{Inference Performance Analysis on NVIDIA RTX 4090.} Kernel latency for key Llama2-7B/13B layers is measured after warm-up iterations and therefore reflects the efficiency of the packed kernel when weights are served from cache, rather than cold-DRAM bandwidth. We also report Llama2-7B end-to-end decoding throughput for a 256-token generation as the primary practical system-level metric. RaBiT shows superior efficiency at both the kernel and system levels.
}
\label{tab:inference_speed}
\small
\resizebox{0.9\textwidth}{!}{
    \begin{tabular}{lc ccccc c}
\toprule
\multirow{2.5}{*}{\textbf{Method}} & \multirow{2.5}{*}{\textbf{Bit}} & \multicolumn{4}{c}{\textbf{Kernel-Level Latency ($\mu$s) $\downarrow$}} &  \textbf{End-to-End Decoding} \\
\cmidrule(lr){3-6}
& & \makecell{4096$\times$4096 \\ \footnotesize(q\_proj, 7B)} & \makecell{11008$\times$4096 \\ \footnotesize(gate\_proj, 7B)} & \makecell{5120$\times$5120 \\ \footnotesize(q\_proj, 13B)} & \makecell{13824$\times$5120 \\ \footnotesize(gate\_proj, 13B)} & \textbf{Throughput (tok/s) $\uparrow$} \\
\midrule
FP16 & 16 & 17.15 (1.00$\times$) & 70.37 (1.00$\times$) & 17.85 (1.00$\times$) & 122.90 (1.00$\times$) & 64.96 (1.00$\times$) \\
\midrule
\multirow{2}{*}{DBF} 
& 2.3 & 12.66 (1.35$\times$) & 28.43 (2.48$\times$) & 14.72 (1.21$\times$) & 31.87 (3.86$\times$) & 157.66 (2.43$\times$) \\
& 2 & 11.47 (1.50$\times$) & 20.90 (3.37$\times$) & 14.08 (1.27$\times$) & 29.58 (4.15$\times$) & 175.21 (2.70$\times$) \\
\midrule
\multirow{2}{*}{QTIP} 
& 3 & 24.04 (0.71$\times$) & 37.08 (1.90$\times$) & 36.22 (0.49$\times$) & 49.97 (2.46$\times$) & 153.59 (2.36$\times$) \\
& 2 & 23.40 (0.73$\times$) & 42.40 (1.66$\times$) & 37.46 (0.48$\times$) & 59.22 (2.08$\times$) & 171.74 (2.64$\times$) \\
\midrule
\multirow{2}{*}{\textbf{RaBiT (Ours)}} 
& 3 & \textbf{8.15 (2.10$\times$)} & \textbf{17.13 (4.11$\times$)} & \textbf{9.90 (1.80$\times$)} & \textbf{22.36 (5.50$\times$)} & \textbf{191.63 (2.95$\times$)} \\
& 2 & \textbf{7.72 (2.22$\times$)} & \textbf{15.71 (4.48$\times$)} & \textbf{8.33 (2.14$\times$)} & \textbf{17.50 (7.02$\times$)} & \textbf{\THROUGHPUT (\SPEEDUP)} \\
\bottomrule
\end{tabular}

}
\end{table*}

\subsection{Ablation Studies}
\label{sec:ablation_study}
\subsubsection{Component-wise Contribution Analysis}
We performed an ablation study to analyze the contributions of RaBiT's core components Coupled QAT, Iterative SVID (\textbf{I}), and I/O Channel Importance-Scaled Preconditioning (\textbf{S}), with results in \Cref{tab:ablation-RaBiT}.
The analysis clearly shows that \textbf{Coupled QAT} is the most critical performance factor. Simply switching from Standard QAT (6.55 PPL) to Coupled QAT reduces the perplexity to 5.84, confirming that resolving inter-path adaptation yields the largest gain.
Our initialization methods (\textbf{I} and \textbf{S}) provide further essential improvements. While they offer a significant boost to the baseline Standard QAT, their role within the powerful Coupled QAT framework is to provide the final, crucial fine-tuning needed to reach the optimal 5.78 PPL. This synergy between a robust training method and a function-aware initialization is key to state-of-the-art performance of RaBiT.

\begin{figure}[t!]
  \centering
  \begin{subfigure}[b]{0.90\linewidth}
    \centering
    \includegraphics[width=\linewidth]{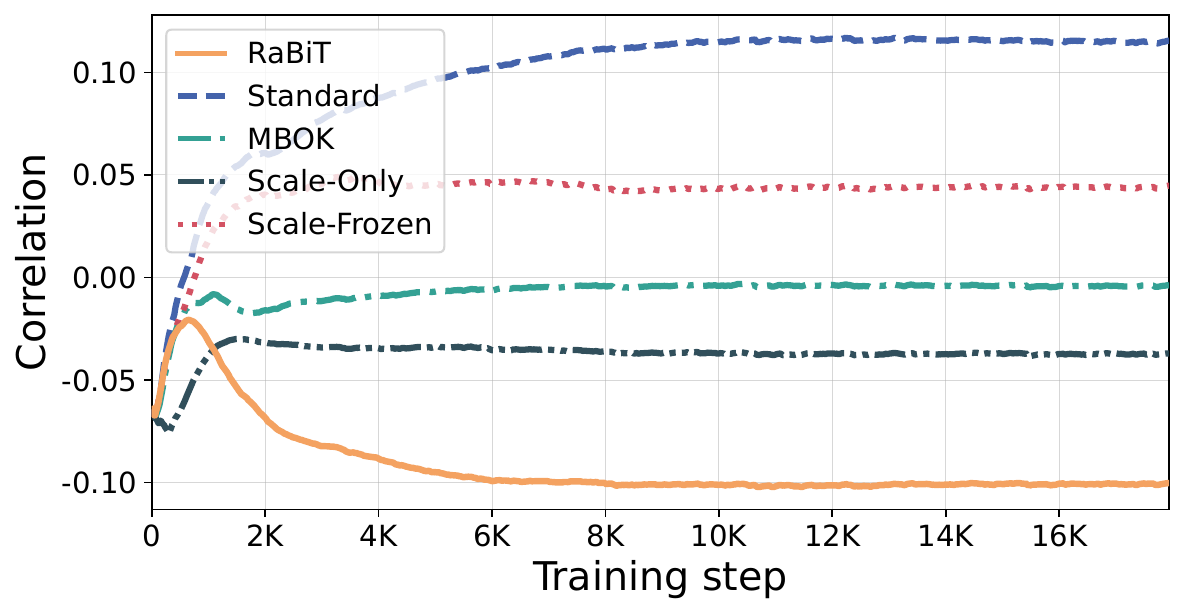}
    \vspace{-15pt}
    \caption{Inter-Path Correlation}
    \label{fig:inter_path_corr}
  \end{subfigure}
  \begin{subfigure}[b]{0.90\linewidth}
    \centering
    \includegraphics[width=\linewidth]{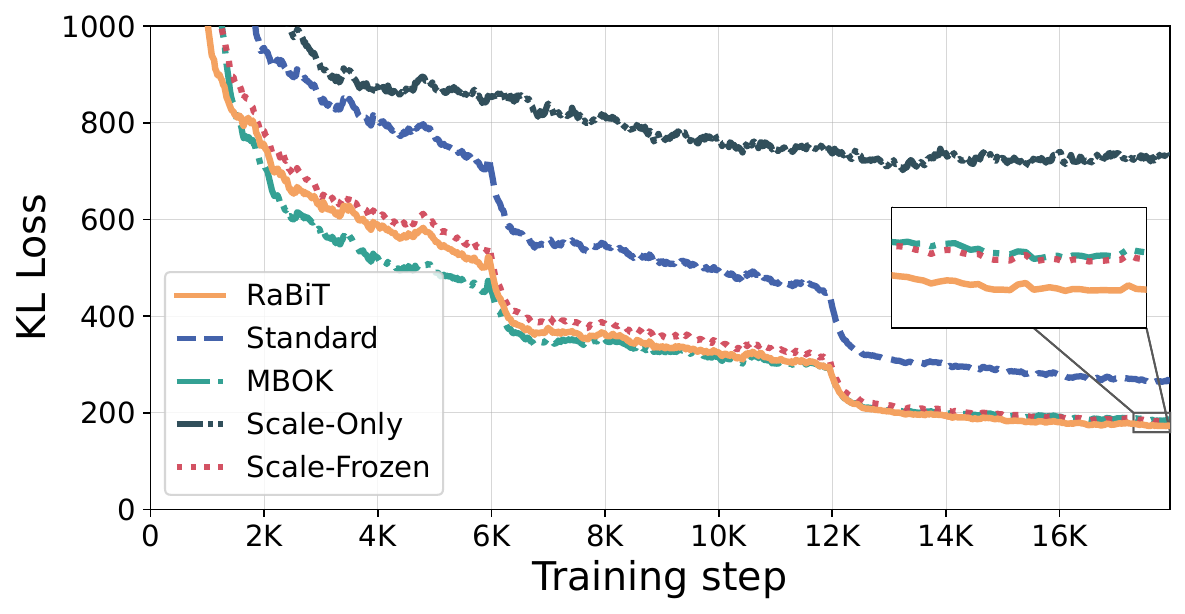}
    \vspace{-15pt}
    \caption{Training Loss}
    \label{fig:training_loss}
  \end{subfigure}
  \caption{\textbf{Visualization of Coupled Training Dynamics.} 
    \textbf{(a) Inter-Path Correlation:} RaBiT enforces a negative inter-path correlation, indicating effective error-correction, whereas Standard QAT leads to positive correlation (co-adaptation). 
    \textbf{(b) Training Loss:} This structural advantage directly translates to a lower and more stable training loss for RaBiT, demonstrating its superior optimization path.}
  \label{fig:corr_and_loss}
  \vspace{-5pt}
\end{figure}

\subsubsection{Analysis of Coupled Training Dynamics}
To empirically validate that coupled training resolves co-adaptation, we conducted a controlled experiment comparing RaBiT to four variants: (1) Standard QAT (independent latent weights), (2) MBOK (frozen primary binary core, mimicking the path-freezing heuristic of \citet{mbok2025} within our training loop), (3) Scale-only (frozen binary cores), and (4) Scale-frozen (RaBiT with frozen scales). To strictly isolate architectural dynamics from optimizer-specific effects, we standardize the initialization and optimizer (Muon) across all variants.

\Cref{fig:corr_and_loss} reveals the resulting training dynamics. As theorized, RaBiT successfully maintains a stable negative inter-path correlation, enforcing the error-correction hierarchy (\Cref{fig:inter_path_corr}). In contrast, Standard QAT develops a strong positive correlation, confirming that a shared global gradient induces harmful redundancy. The constrained variants (MBOK, Scale-frozen) fail to establish a strong anti-correlation, limiting their optimization potential. This structural advantage directly translates to model functionality, as shown by the training loss curves (\Cref{fig:training_loss}). RaBiT achieves the lowest and most stable loss, while the co-adaptation in Standard QAT and the incomplete optimization of the other variants lead to significantly higher loss. This analysis confirms that sequential optimization of all parameters algorithmically prevents co-adaptation and results in its superior performance.

\subsection{Inference Performance}
\label{subsec:hw_efficiency}
RaBiT not only achieves state-of-the-art accuracy but also delivers exceptional inference speed by leveraging its parallelizable matmul-free binary architecture. As shown in \Cref{tab:inference_speed}, 2-bit RaBiT achieves a \SPEEDUP speed-up in end-to-end decoding throughput over the FP16 baseline, on an NVIDIA RTX 4090.

This performance gain stems from two key advantages. First, the 8$\times$ reduction in model size (2-bit vs. 16-bit) dramatically lowers memory bandwidth requirements, which is the primary bottleneck in the autoregressive decoding phase. Second, unlike VQ methods, RaBiT avoids hardware-unfriendly overheads like lookup tables or rotations. Its simple architecture of additions and element-wise scaling translates to higher hardware utilization. This is evident in our kernel-level benchmarks, where RaBiT's specialized kernels exhibit consistently lower latency than both the FP16 baseline and QTIP's VQ kernels. By eliminating computational complexity, RaBiT ensures that theoretical memory savings translate directly into real-world speed, delivering a solution that is both accurate and genuinely efficient. Further details on our kernel design and additional performance benchmarks are provided in \Cref{appendix:kernel_design} and \Cref{appendix:more_comparisons}, respectively.

% ===================== SECTION 6: CONCLUSION =====================
\section{Conclusion}
\label{sec:conclusion}
This study addresses the trade-off between accuracy and hardware efficiency in 2-bit large language model (LLM) quantization through the introduction of RaBiT. We identified inter-path adaptation as a primary bottleneck that compromises the error-compensation structure in residual binarization. To mitigate this, the proposed framework employs on-the-fly residual coupling, a mechanism that prevents structural breakdown during training and ensures that the expressive capacity of the model is effectively used. Furthermore, to address the instability inherent in extreme quantization, we introduced a function-aware initialization strategy that facilitates stable convergence. Experimental results indicate that RaBiT achieves state-of-the-art performance at 2-bit precision, outperforming both existing binary methods and hardware-intensive Vector Quantization (VQ) approaches. By establishing these capabilities, this work offers a viable pathway for the efficient deployment of high-performance low-bit LLMs and serves as a scalable foundation for future research.

% % Acknowledgements should only appear in the accepted version.
\section*{Acknowledgments}

We thank Drs. Heonjae Ha and SangJeong Lee for their guidance, support and valuable feedback throughout this project.

\section*{Impact Statement}

This paper presents RaBiT, a framework for the extreme 2-bit quantization of LLMs. By drastically reducing computational requirements, our work is environmentally friendly and democratizes access, enabling high-performance models to run on consumer hardware (\eg RTX 4090). This efficiency facilitates local deployment, thereby enhancing user privacy and data sovereignty compared to cloud-based solutions. However, such accessibility also introduces dual-use risks by potentially enabling malicious applications outside controlled environments. Additionally, while our method preserves general capabilities, the specific effects of extreme compression on safety alignment remain an important area for future investigation.

\bibliography{icml2026}
\bibliographystyle{icml2026}

%%%%%%%%%%%%%%%%%%%%%%%%%%%%%%%%%%%%%%%%%%%%%%%%%%%%%%%%%%%%%%%%%%%%%%%%%%%%%%%
%%%%%%%%%%%%%%%%%%%%%%%%%%%%%%%%%%%%%%%%%%%%%%%%%%%%%%%%%%%%%%%%%%%%%%%%%%%%%%%
% APPENDIX
%%%%%%%%%%%%%%%%%%%%%%%%%%%%%%%%%%%%%%%%%%%%%%%%%%%%%%%%%%%%%%%%%%%%%%%%%%%%%%%
%%%%%%%%%%%%%%%%%%%%%%%%%%%%%%%%%%%%%%%%%%%%%%%%%%%%%%%%%%%%%%%%%%%%%%%%%%%%%%%
\newpage
\appendix
\onecolumn

\crefalias{section}{appendix}
\crefalias{subsection}{appendix}
\crefname{appendix}{Appendix}{Appendices}

\section{Mathematical Analysis of Training Dynamics}
\label{appendix:proofs}

This section provides a mathematical analysis of the training dynamics for residual binary architectures. We demonstrate why Standard QAT is prone to \textbf{inter-path adaptation}, where paths become redundant. In contrast, we show how RaBiT's coupled training mechanism structurally enforces an \textbf{error-correcting hierarchy} and is superior to other heuristic solutions.

\begin{proposition}[Inter-Path Adaptation in Standard QAT]
\label{prop:independent_qat_tendency}
In a Standard QAT scheme where two paths ($\hat{\W}_1, \hat{\W}_2$) are updated from their respective latent weights ($\W_1, \W_2$) using a shared global gradient $\mathbf{G} = \nabla_{\hat{\W}_1+\hat{\W}_2}\mathcal{L}$, the paths have a persistent tendency to become positively correlated, leading to redundancy.
\end{proposition}
\begin{proof}
Let the latent weights be $\W_1$ and $\W_2$. After a single update step with learning rate $\eta$ and shared gradient $\mathbf{G}$, the new weights are $\W_1'$ and $\W_2'$:
$$ \W_1' := \W_1 - \eta \mathbf{G} \quad \text{and} \quad \W_2' := \W_2 - \eta \mathbf{G} $$
The change in the Frobenius inner product between the weights, which reflects their correlation, is:
$$ \Delta_{\langle\cdot,\cdot\rangle} := \langle \W_1', \W_2' \rangle_F - \langle \W_1, \W_2 \rangle_F $$
Expanding this gives:
$$ \Delta_{\langle\cdot,\cdot\rangle} = -\eta \left( \langle \W_1, \mathbf{G} \rangle_F + \langle \W_2, \mathbf{G} \rangle_F \right) + \eta^2 \|\mathbf{G}\|_F^2 $$
While the linear terms depend on the alignment between the current weights and the gradient, the quadratic term $\eta^2 \|\mathbf{G}\|_F^2$ is \textbf{always non-negative}. This term acts as a systematic force, constantly pushing the two paths in the same direction defined by the global gradient $\mathbf{G}$. This dynamic, the underlying mechanism of \textbf{inter-path adaptation}, compels both paths to learn redundant, dominant features in order to minimize the global loss. This ultimately leads to a breakdown of the intended residual hierarchy and compromises the model's expressive capacity.
\end{proof}

\begin{proposition}[Structurally Enforced Error Correction in RaBiT]
\label{prop:coupled_derivation}
RaBiT's coupled training mechanism resolves the redundancy drift by fundamentally changing the optimization objective. Instead of independent updates, RaBiT's on-the-fly derivation structurally forces the second path ($\hat{\W}_2$) to align with the true residual of the first path ($\R_1 = \W_{\mathrm{FP}} - \hat{\W}_1$), thereby enforcing an error-correcting relationship.
\end{proposition}
\begin{analysis}
The optimization objectives of the two paths are implicitly different in RaBiT versus the naïve approach.

\begin{itemize}[itemsep=1.5pt, topsep=0pt, leftmargin=*]
    \item \textbf{Standard QAT Objective}: Both paths are driven by the same structurally-agnostic global gradient $\mathbf{G}$. Their implicit goal is to align with $\mathbf{G}$ to reduce the global loss. Since both $\hat{\W}_1$ and $\hat{\W}_2$ are incentivized to align with the same vector $\mathbf{G}$, they inevitably learn to align with each other, leading to redundancy as shown in Proposition 1.
    $$ \hat{\W}_1 \propto \mathbf{G} \quad \text{and} \quad \hat{\W}_2 \propto \mathbf{G} \implies \langle \hat{\W}_1, \hat{\W}_2 \rangle_F > 0 $$
    
    \item \textbf{RaBiT's Enforced Objective}: RaBiT maintains a single shared blueprint, $\W_{\mathrm{FP}}$. The on-the-fly derivation process, $\R_1 := \W_{\mathrm{FP}} - \hat{\W}_1$ followed by the binarization of $\R_1$ to create $\hat{\W}_2$, explicitly defines the optimization target for the second path. The goal for $\hat{\W}_2$ is no longer to align with the global gradient $\g$, but to be the best possible low-rank approximation of the current residual $\R_1$.
    $$ \text{Objective for } \hat{\W}_2: \quad \min \|\R_1 - \hat{\W}_2\|_F^2 \implies \hat{\W}_2 \approx \R_1 $$
    This structural constraint forces a high \textbf{Residual Alignment}. In the context of extreme low-bit quantization, the first approximation $\hat{\W}_1$ often "overshoots" the target $\W_{\mathrm{FP}}$ in certain directions. To correct this, the residual $\R_1 = \W_{\mathrm{FP}} - \hat{\W}_1$ will point in the opposite direction of the overshoot. By aligning with $\R_1$, $\hat{\W}_2$ naturally becomes \textbf{anti-correlated} with $\hat{\W}_1$, implementing an efficient \textbf{active cancellation} mechanism rather than degenerating into redundancy.
\end{itemize}
\end{analysis}

\begin{proposition}[Superior Optimization Dynamics of Coupled vs. Iterative Training]
\label{prop:iterative_suboptimality}
Iterative training (\eg freezing one path while training the other) avoids adaptation but at the cost of optimization efficiency. In contrast, RaBiT resolves adaptation while permitting full parameter co-adaptation, resulting in a superior optimization trajectory.
\end{proposition}
\begin{proof}
Following Proposition 1, the problem of Standard QAT is the simultaneous update of both paths in the same direction. An alternative solution is to update them iteratively, which prevents this simultaneous push and thus avoids adaptation. However, this introduces a new problem of inefficiency.

The optimal direction to reduce the loss $\mathcal{L}$ is the steepest descent direction in the joint parameter space of $(\W_1, \W_2)$, which is $\mathbf{d}^* = (-\mathbf{G}, -\mathbf{G})$. When training iteratively, one path is frozen, so the update is restricted to an axis-aligned direction, \eg $\mathbf{d}_{\text{iter}} = (\mathbf{0}, -\mathbf{G})$. The cosine similarity between the iterative update and the optimal update direction is:
\begin{align*}
    \cos(\theta) &= \frac{\langle \mathbf{d}_{\text{iter}}, \mathbf{d}^* \rangle_F}{\|\mathbf{d}_{\text{iter}}\|_F \|\mathbf{d}^*\|_F} 
    = \frac{\langle (\mathbf{0}, -\mathbf{G}), (-\mathbf{G}, -\mathbf{G}) \rangle_F}{\|(\mathbf{0}, -\mathbf{G})\|_F \|(-\mathbf{G}, -\mathbf{G})\|_F} \\
    &= \frac{\|\mathbf{G}\|_F^2}{\|\mathbf{G}\|_F \cdot \sqrt{\|\mathbf{G}\|_F^2 + \|\mathbf{G}\|_F^2}} = \frac{1}{\sqrt{2}}
\end{align*}
This fixed $45^\circ$ misalignment forces the optimization to follow an inefficient zig-zag trajectory. While it solves adaptation, it sacrifices optimization efficiency.

RaBiT, through its coupled derivation described in Proposition 2, resolves this trade-off. By updating a single shared weight $\W_{\mathrm{FP}}$ with the full gradient $\mathbf{G}$, it allows both paths to co-adapt simultaneously in a coordinated manner that is not restricted to an inefficient path. Thus, RaBiT resolves adaptation without compromising optimization efficiency, leading to superior dynamics.
\end{proof}

\paragraph{Corollary 1 (to Proposition 2). Negative Correlation Induction.}
\label{col:neg_corr_induction}
\textit{RaBiT's coupled training mechanism, by forcing the second path ($\hat{\W}_2$) to approximate the residual of the first path ($\R_1$), inherently promotes a negative correlation between their respective outputs ($\y_1, \y_2$).}

\begin{analysis}
From Proposition 2, we established that RaBiT trains the second path to approximate the residual of the first:
$$
    \hat{\W}_2 \approx \R_1 = \W_{\mathrm{FP}} - \hat{\W}_1
$$
Let us consider the outputs for a given input $\x$. The outputs of the full-precision teacher, the first path, and the second path are $\y_t = \W_{\mathrm{FP}}\x$, $\y_1 = \hat{\W}_1\x$, and $\y_2 = \hat{\W}_2\x$, respectively. Based on the weight approximation, the output of the second path is:
$$
    \y_2 \approx \R_1\x = (\W_{\mathrm{FP}} - \hat{\W}_1)\x = \y_t - \y_1
$$
Now, we can analyze the covariance between the outputs $\y_1$ and $\y_2$. Assuming the outputs are centered for simplicity, the covariance is proportional to the expected value of their dot product, $\mathbb{E}[\y_1^\top \y_2]$.
$$
    \mathbb{E}[\y_1^\top \y_2] \approx \mathbb{E}[\y_1^\top (\y_t - \y_1)] = \mathbb{E}[\y_1^\top \y_t] - \mathbb{E}[\y_1^\top \y_1] = \mathbb{E}[\y_1^\top \y_t] - \mathbb{E}[\|\y_1\|^2]
$$
Let us analyze the two terms:
\begin{enumerate}[itemsep=1pt, topsep=2pt]
    \item $\mathbb{E}[\y_1^\top \y_t]$: The first path $\hat{\W}_1$ is the primary, albeit coarse, approximation of $\W_{\mathrm{FP}}$.
    Its purpose is to capture the main features of the teacher, so $\y_1$ and $\y_t$ are expected to be \textbf{positively correlated}.
    However, in the extreme 1-bit regime, the approximation is directionally coarse: a non-negligible portion of $\y_1$ lies in directions \emph{not} well-aligned with $\y_t$.
    This misaligned component limits how large the alignment term $\y_1^\top \y_t$ can become, making the positive correlation typically \emph{not} strong.

    \item $\mathbb{E}[\|\y_1\|^2]$: This is the expected squared norm of the first path's output. Binarization is an aggressive quantization that often leads to an ``overshoot'' in magnitude. A single binary path must represent a wide range of continuous values, so its effective scaling factor often results in an output magnitude $\|\y_1\|$ that exceeds the projection of $\y_1$ onto $\y_t$. Consequently, it is often the case that $\|\y_1\|^2 > \y_1^\top \y_t$, making $\mathbb{E}[\|\y_1\|^2]$ a larger positive term than $\mathbb{E}[\y_1^\top \y_t]$.
\end{enumerate}
Combining these points, the covariance is approximately the difference between a positive term and a larger positive term:
$$
    \operatorname{Cov}(\y_1, \y_2) \approx \underbrace{\mathbb{E}[\y_1^\top \y_t]}_{\text{Positive Alignment}} - \underbrace{\mathbb{E}[\|\y_1\|^2]}_{\text{Larger Magnitude Term}} < 0
$$
Thus, RaBiT's mechanism of forcing the second path to correct the error of the first path structurally drives the covariance, and therefore the correlation $\operatorname{Corr}(\y_1, \y_2)$, toward negative, consistent with our empirical observations.
\end{analysis}

\section{Extended Analysis: Inter-Path Adaptation under KL Divergence}
\label{appendix:kl_analysis}

While Section 3 and Corollary 1 motivate RaBiT using the Mean Squared Error (MSE) decomposition, modern LLM training often relies on the Kullback-Leibler (KL) divergence loss. In this section, we rigorously demonstrate that the principle of ``inter-path adaptation'' and RaBiT's solution remain valid under the KL divergence objective. We leverage a local quadratic approximation and the findings from \citep{kim2021comparing} regarding the decomposition of KL loss.

\begin{proposition}[Optimality of Residual Coupling under KL Divergence]
RaBiT's residual coupling mechanism structurally eliminates the optimization bias inherent in the KL divergence loss by enforcing a negative Hessian-weighted path correlation, a property that Standard QAT fails to satisfy.
\end{proposition}

\begin{analysis}
We analyze the optimization dynamics at a specific linear layer $\ell$ where RaBiT is applied. Let $\y^\ell_t = \W_{\mathrm{FP}} \x^\ell$ be the output feature of the teacher model, and $\y^\ell_s$ be that of the student. Since the operations within the layer are linear, the student's output is the sum of its binary paths: $\y^\ell_s = \mathbf{h}_1 + \mathbf{h}_2$.

\paragraph{Local Quadratic Approximation.}
The global KL divergence loss $\mathcal{L}_{KL}$ can be approximated locally around the teacher's output $\y^\ell_t$ using a second-order Taylor expansion:
\begin{equation}
    \mathcal{L}_{KL}(\y^\ell_s) \approx \mathcal{L}_{KL}(\y^\ell_t) + \nabla \mathcal{L}(\y^\ell_t)^\top \Delta \y + \frac{1}{2} \Delta \y^\top \mathbf{H}^\ell \Delta \y
\end{equation}
Assuming the teacher is optimal locally ($\nabla \mathcal{L} \approx 0$), the optimization objective reduces to minimizing a \textbf{Hessian-weighted MSE}:
\begin{equation}
    \mathcal{J}_{local} \approx \frac{1}{2} \| \y^\ell_s - \y^\ell_t \|_{\mathbf{H}^\ell}^2 = \frac{1}{2} (\mathbf{h}_1 + \mathbf{h}_2 - \y^\ell_t)^\top \mathbf{H}^\ell (\mathbf{h}_1 + \mathbf{h}_2 - \y^\ell_t)
\end{equation}
where $\mathbf{H}^\ell$ is the Hessian matrix representing the local curvature.

\paragraph{Decomposition with Hessian-weighted Path Correlation.}
Similar to \Cref{eq:mse_decomp} in the main text, we decompose this local objective. By treating $(\mathbf{h}_1 - \y^\ell_t)$ as the residual error of the first path, we expand the quadratic term:
\begin{equation}
    \mathcal{J}_{local} \propto \underbrace{\| \mathbf{h}_1 - \y^\ell_t \|_{\mathbf{H}^\ell}^2}_{\text{Base Error}} + \underbrace{\| \mathbf{h}_2 \|_{\mathbf{H}^\ell}^2}_{\text{Path 2 Amp.}} + \underbrace{2 (\mathbf{h}_1 - \y^\ell_t)^\top \mathbf{H}^\ell \mathbf{h}_2}_{\text{Interaction Term}}
\end{equation}
We define the \textit{Hessian-weighted Path Correlation}, denoted as $\text{PathCorr}_{\mathbf{H}}$, based on the generalized cosine similarity in the inner product space defined by $\mathbf{H}^\ell$:
\begin{equation}
    \text{Interaction Term} = 2 \cdot \| \mathbf{h}_1 - \y^\ell_t \|_{\mathbf{H}^\ell} \| \mathbf{h}_2 \|_{\mathbf{H}^\ell} \cdot \textbf{PathCorr}_{\mathbf{H}}(\mathbf{h}_1 - \y^\ell_t, \mathbf{h}_2)
\end{equation}
This formulation reveals that minimizing the local loss requires maximizing the negative magnitude of $\text{PathCorr}_{\mathbf{H}}$.

\paragraph{Bridging Interaction to Bias Cancellation.}
To understand the implication of this interaction term in the context of KL divergence, we refer to \citep{kim2021comparing}, which proved that minimizing $\mathcal{L}_{KL}$ is equivalent to minimizing MSE plus a negative \textbf{Bias Term ($\delta$)}. This bias term acts as a destabilizing force that pushes the student's logit sum to diverge from the teacher's. Crucially, effectively maximizing the negative interaction term (\ie enforcing error correction) is the key to neutralizing this bias.

\begin{itemize}[leftmargin=*]
    \item \textbf{Failure of Standard QAT:} In Standard QAT, $\mathbf{h}_1$ and $\mathbf{h}_2$ are updated independently using the same gradient signal. This leads to $\mathbf{h}_2$ aligning with $\y^\ell_t$ rather than the residual $(\y^\ell_t - \mathbf{h}_1)$. Consequently, $\text{PathCorr}_{\mathbf{H}}$ becomes positive (redundancy) or near zero. This failure to exploit the interaction bonus leaves the destabilizing Bias Term ($\delta$) unchecked, causing the optimization drift described in \citep{kim2021comparing}.

    \item \textbf{Success of RaBiT:} RaBiT enforces $\hat{\W}_2 = \operatorname{sign}(\W_{\mathrm{FP}} - \hat{\W}_1)$, structurally guaranteeing:
    \begin{equation}
        \mathbf{h}_2 \approx \y^\ell_t - \mathbf{h}_1 = -(\mathbf{h}_1 - \y^\ell_t)
    \end{equation}
    This forces the second path vector to be anti-parallel to the first path's error vector in the feature space. As a result:
    \begin{enumerate}
        \item \textbf{Maximized Interaction:} $\text{PathCorr}_{\mathbf{H}}$ approaches its theoretical minimum of $-1$ (perfect anti-correlation).
        \item \textbf{Bias Cancellation:} By strictly adhering to the residual $\mathbf{h}_2 \approx \y^\ell_t - \mathbf{h}_1$, the total student output $\y^\ell_s$ approximates $\y^\ell_t$ without the scale divergence issues. The Bias Term ($\delta$) is effectively cancelled out locally ($(\sum \y^\ell_s - \sum \y^\ell_t)^2 \approx 0$).
    \end{enumerate}
\end{itemize}
Therefore, RaBiT's residual coupling is the optimal strategy for minimizing $\mathcal{L}_{KL}$ as it structurally enforces the necessary negative correlation that Standard QAT fails to learn.
\end{analysis}

\section{Initialization Analysis: Functionality vs. Approximation}

\begin{table*}[ht]
  \centering
    \caption{\textbf{Initialization Analysis on Llama2-7B.} Trade-off between weight reconstruction error (Avg. MAE/MSE) and model functionality (Initial KL Divergence Loss), for the first \texttt{q\_proj} layer. I/O Channel Importance Scaling dramatically reduces KL Divergence Loss despite increasing MSE.}
     
  \label{tab:initialization_analysis}
  \small
  \setlength{\tabcolsep}{4.5pt}
  \renewcommand{\arraystretch}{1.12}
  \resizebox{0.60\textwidth}{!}{
    \begin{tabular}{@{}lccc@{}}
\toprule
\textbf{Initialization Method} & \textbf{Avg. MAE $\downarrow$} & \textbf{Avg. MSE $\downarrow$} & \textbf{KL Loss $\downarrow$} \\
\midrule
Greedy SVID                 & 0.359 & 0.150 & 17,152 \\
Iterative Residual SVID     & 0.370 & 0.122 & 13,760 \\
\hspace{1em}+ I/O Ch. Importance Scaling    & \textbf{0.632} & \textbf{0.302} & \textbf{2,672} \\
\bottomrule
\end{tabular}
  }
\end{table*}

\begin{figure}[ht!]
    \centering
    \captionsetup[subfigure]{justification=centering}

    % --- First Row ---
    \begin{subfigure}[b]{0.24\textwidth}
        \centering
        \includegraphics[width=\linewidth]{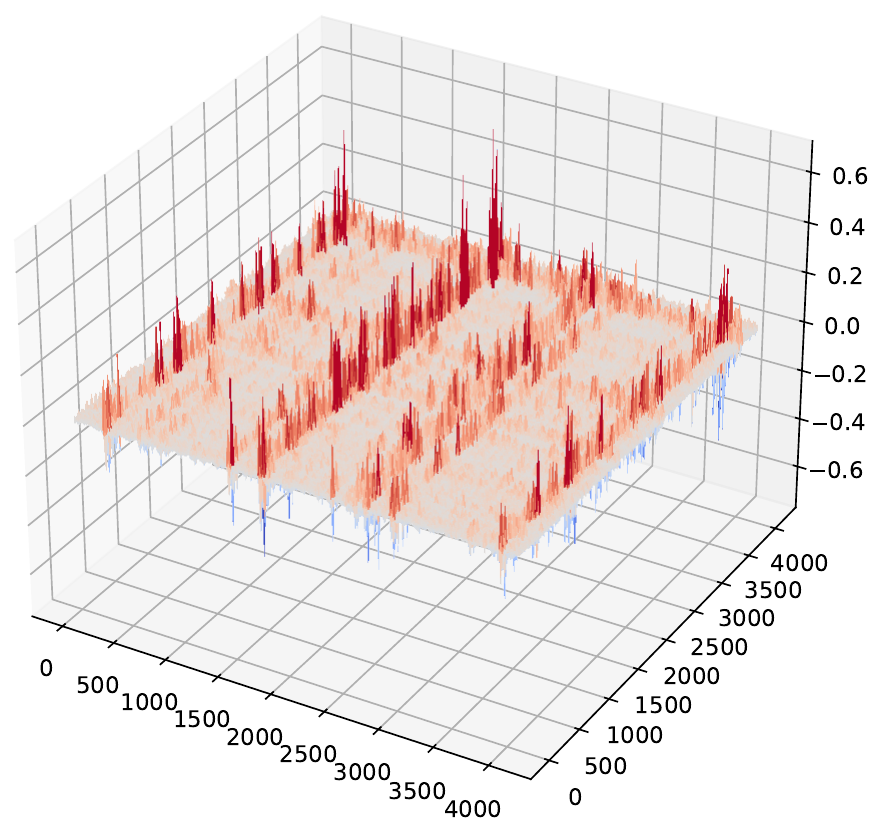}
        \caption{$\W_{\mathrm{FP}}$}
        \label{fig:init-wfp}
    \end{subfigure}%
    \hfill
    \begin{subfigure}[b]{0.24\textwidth}
        \centering
        \includegraphics[width=\linewidth]{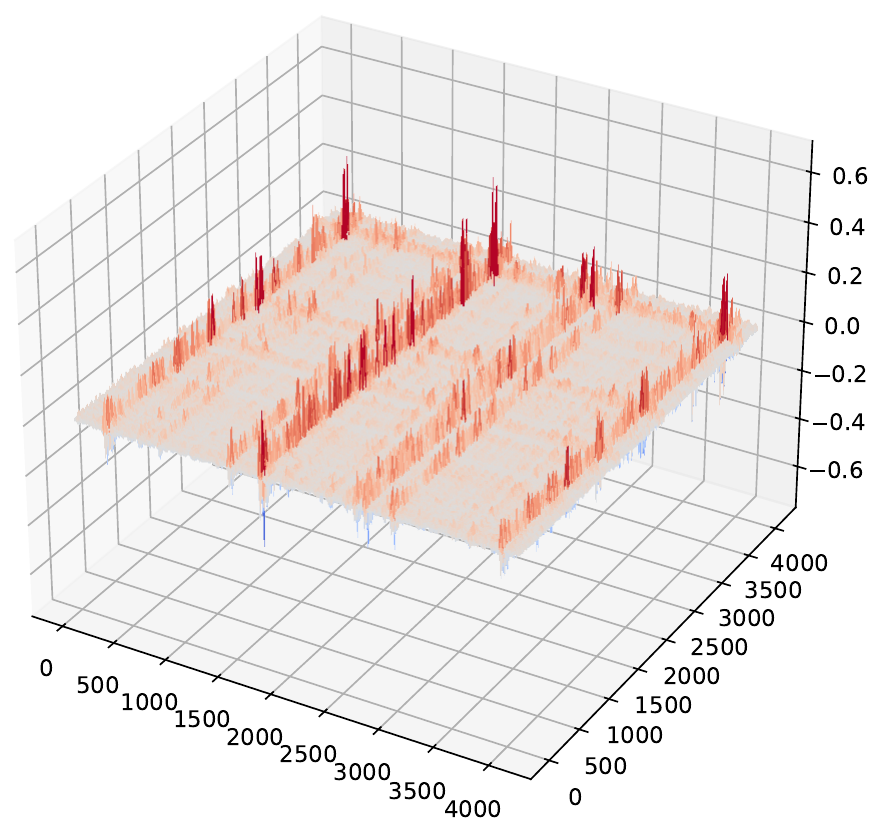}
        \caption{Greedy SVID Init.}
        \label{fig:init-greedy}
    \end{subfigure}%
    \hfill
    \begin{subfigure}[b]{0.24\textwidth}
        \centering
        \includegraphics[width=\linewidth]{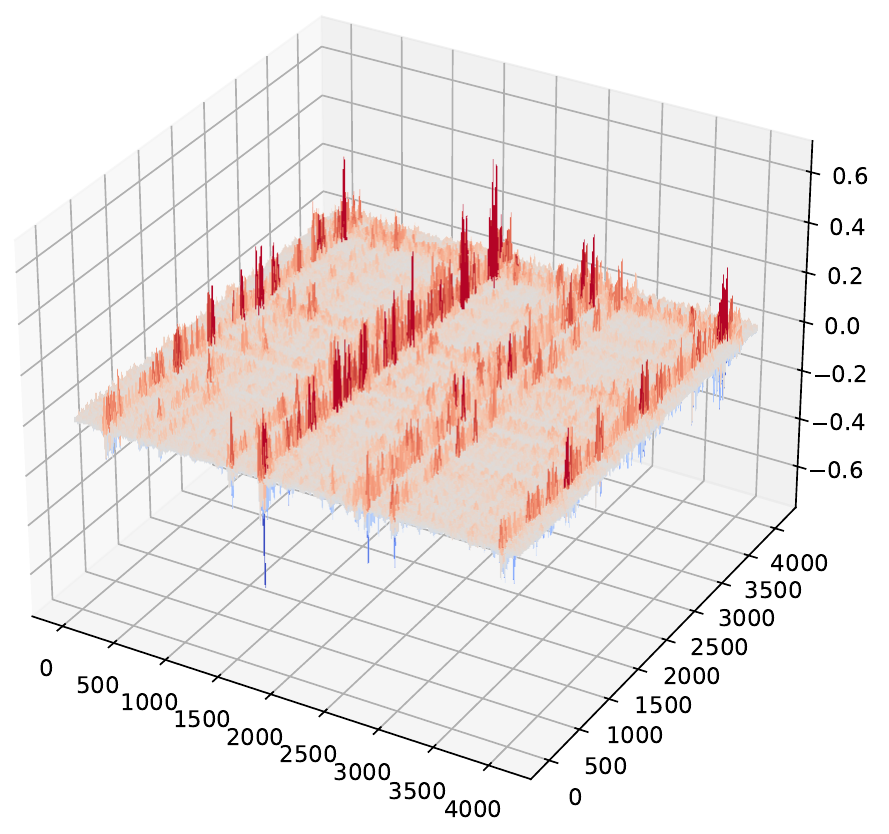}
        \caption{Iterative SVID Init.}
        \label{fig:init-iterative}
    \end{subfigure}%
    \hfill
    \begin{subfigure}[b]{0.24\textwidth}
        \centering
        \includegraphics[width=\linewidth]{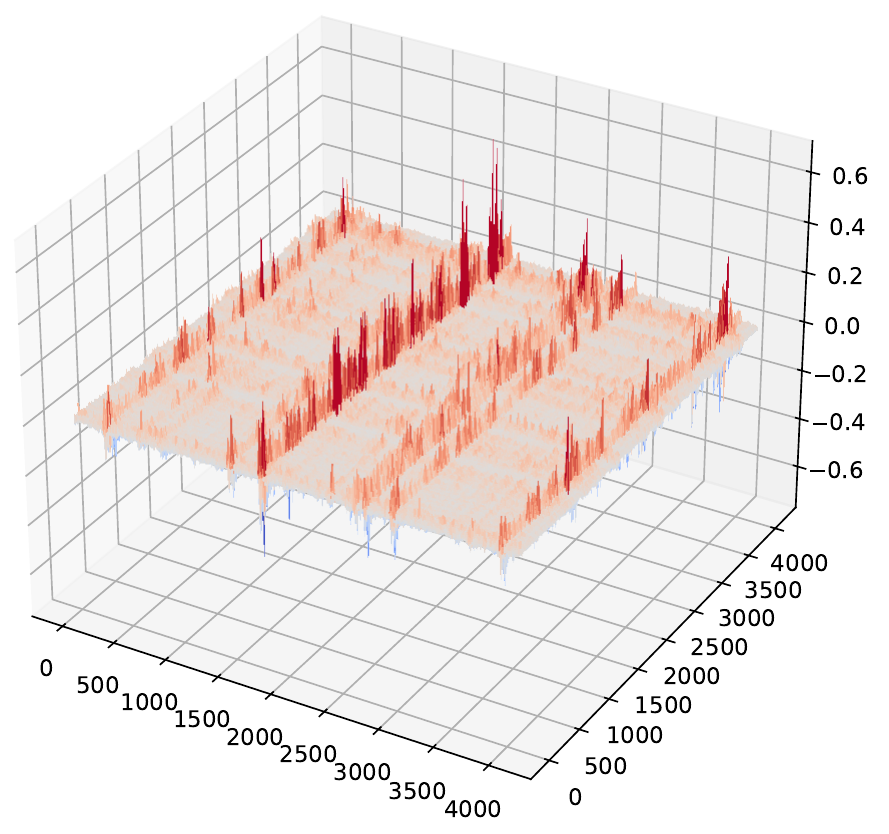}
        \caption{Iter + I/O Scaling}
        \label{fig:init-io}
    \end{subfigure}

    \vspace{3mm} 

    % --- Second Row ---
    \begin{subfigure}[b]{0.24\textwidth}
        \phantom{\includegraphics[width=\linewidth]{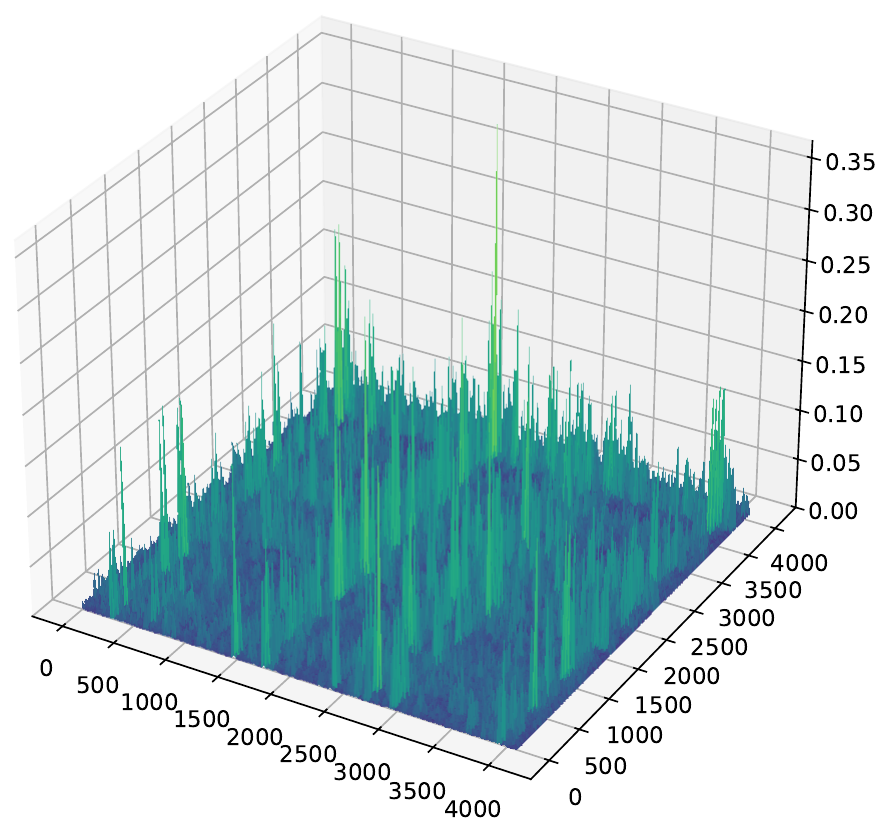}}
    \end{subfigure}%
    \hfill
    \begin{subfigure}[b]{0.24\textwidth}
        \centering
        \includegraphics[width=\linewidth]{figures/weight_dist/q_diff_greedy.pdf}
        \caption{Diff. (Greedy)}
        \label{fig:init-diff-greedy}
    \end{subfigure}%
    \hfill
    \begin{subfigure}[b]{0.24\textwidth}
        \centering
        \includegraphics[width=\linewidth]{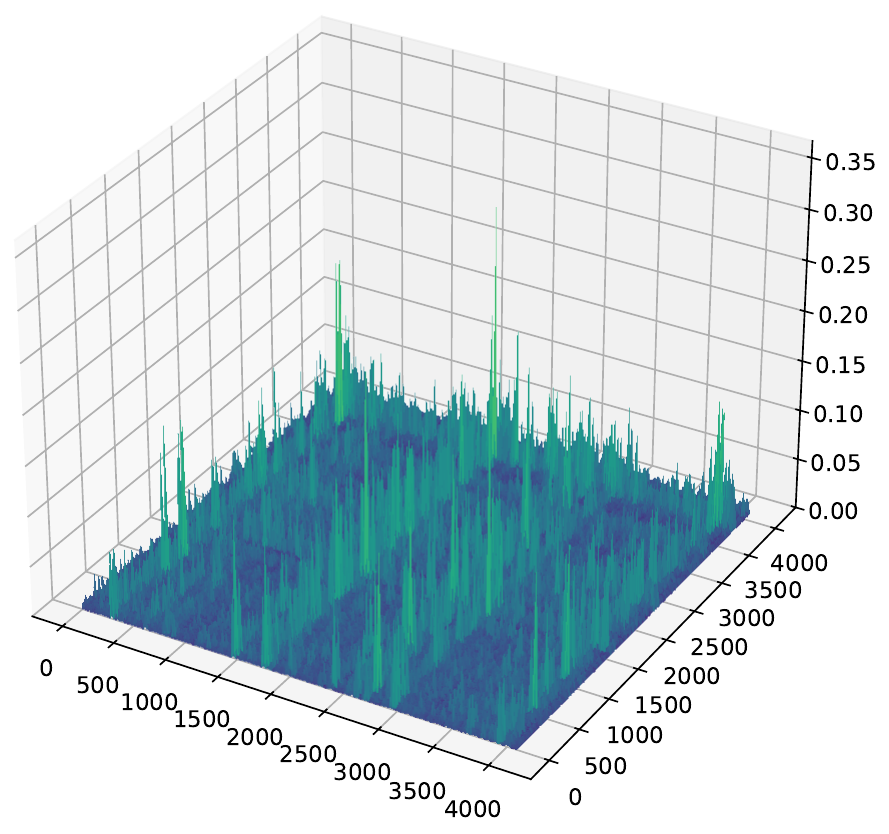}
        \caption{Diff. (Iterative)}
        \label{fig:init-diff-iterative}
    \end{subfigure}%
    \hfill
    \begin{subfigure}[b]{0.24\textwidth}
        \centering
        \includegraphics[width=\linewidth]{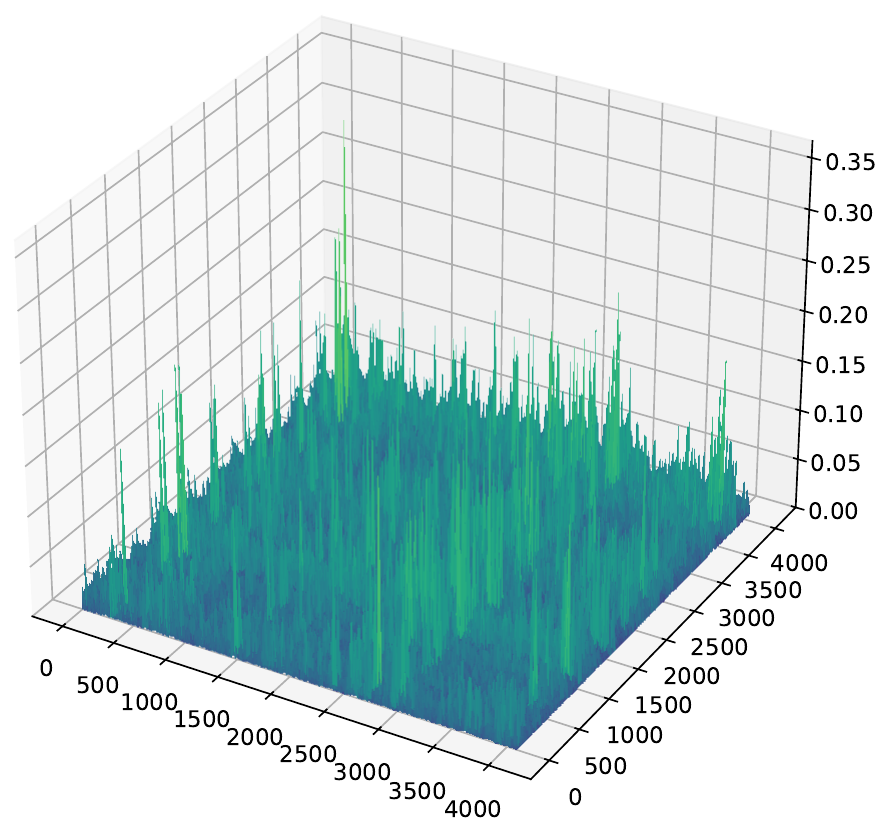}
        \caption{Diff. (Iter + I/O)}
        \label{fig:init-diff-io}
    \end{subfigure}

    \caption{
        \textbf{Visual analysis of weight initialization for the first layer's \texttt{q\_proj} matrix of the Llama2-7B model.}
        The top row displays the original full-precision weight ($\W_{\mathrm{FP}}$) alongside its initial approximations from three methods: (b) Greedy SVID, (c) Iterative SVID, and (d) Iterative SVID with I/O Channel Importance Scaling.
        The bottom row shows the corresponding difference matrices ($\W_{\mathrm{FP}} - \hat{\W}_{\mathrm{init}}$), illustrating the initial error structure.
        Our function-aware initialization produces a qualitatively different structure compared to the others, which is reflected in its distinct error pattern.
    }
    \label{fig:initialization_analysis}
\end{figure}

Stable initialization is paramount in the low-bit regime, as the initial quantization error spike can destabilize QAT. On the Llama2-7B model, we evaluate our proposed techniques--Iterative Residual SVID and I/O Channel Importance Scaling (\Cref{subsec:stable_init})--by measuring both the weight reconstruction error (Avg. MAE, MSE) and the initial task loss (Knowledge Distillation (KD) loss) before the first training step.

\Cref{tab:initialization_analysis} details the results for the first \texttt{q\_proj} layer and reveals a crucial insight. As our baseline, Greedy SVID is a non-iterative decomposition that finalizes each path sequentially without the co-adaptation enabled by our iterative approach. First, regarding \textbf{Iterative Refinement}, moving from Greedy SVID to Iterative Residual SVID consistently improves weight reconstruction (\eg Avg. MSE drops 0.150 $\to$ 0.122) and substantially reduces the initial KL divergence loss (17,152 $\to$ 13,760), confirming mitigation of scheduling bias. Second, adding \textbf{I/O Channel Importance Scaling} to the iterative process yields a striking result: while reconstruction error increases significantly (Avg. MSE 0.122 $\to$ 0.302), the KL divergence loss \textit{plummets} dramatically (13,760 $\to$ 2,672, an 81\% reduction).

This confirms that extreme quantization should prioritize preserving \textit{functionality} over merely approximating \textit{weights}. I/O Channel Importance Scaling allocates the limited 2-bit capacity to critical channels based on activation and gradient statistics (\Cref{subsec:stable_init}), sacrificing the reconstruction of less important weights. This trade-off is visually stark in \Cref{fig:initialization_analysis}. While Iterative SVID produces a lower-error approximation than Greedy SVID (comparing \Cref{fig:init-diff-iterative} to \Cref{fig:init-diff-greedy}), the function-aware I/O Scaling method yields a visibly larger reconstruction error (\Cref{fig:init-diff-io}). Despite this higher weight-level discrepancy, its focus on functional saliency provides a far superior starting point for QAT, as evidenced by the dramatic reduction in initial task loss.

\section{Extended Results}
\label{appendix:ext_results}

\paragraph{Detailed Zero-Shot Reasoning Accuracy.}
\Cref{tab:zero_shot_main_detail_appendix} and \Cref{tab:zero_shot_gemma3_detail_appendix} provide a detailed breakdown of the zero-shot reasoning accuracy across five common benchmarks, complementing the average scores reported in the main text.

\begin{table*}[ht]
  \centering
  \caption{\textbf{Detailed Zero-Shot Reasoning Accuracy on Llama Models} (\%). Comparison of FP16 against leading 2-bit methods on five common benchmarks.}
\resizebox{0.90\textwidth}{!}{
  \setlength{\tabcolsep}{5.5pt}
    \begin{tabular}{llcccccc}
    \toprule
    \textbf{Models} & \textbf{Method} & \textbf{WinoGrande$\uparrow$} & \textbf{HellaSwag$\uparrow$} & \textbf{ARC-e$\uparrow$} & \textbf{ARC-c$\uparrow$}  & \textbf{PIQA$\uparrow$}   & \textbf{Average$\uparrow$} \\
    \midrule
    \multirow{4}[2]{*}{Llama2-7B} & FullPrecision & 67.80  & 56.71  & 69.28  & 39.93  & 78.29  & 62.40  \\
          & QTIP & 64.64  & 53.09  & 65.57  & 35.67  & \textbf{75.90}  & 58.97  \\
          & DBF & 63.61  & 52.44  & 64.73  & 35.58  & 75.84  & 58.44  \\
          & RaBiT (Ours) & \textbf{67.80}  & \textbf{53.52}  & \textbf{72.43}  & \textbf{37.88}  & \textbf{75.90}  & \textbf{61.51}  \\
    \midrule
    \multirow{4}[2]{*}{Llama2-13B} & FullPrecision & 69.93  & 59.64  & 73.19  & 45.73  & 78.67  & 65.43 \\
          & QTIP & \textbf{67.56}  & \textbf{57.4}  & \textbf{70.8}  & \textbf{41.46}  & 77.37  & \textbf{62.92}  \\
          & DBF & 67.09  & 56.6  & 69.02  & 38.74  & \textbf{78.18}  & 61.93  \\
          & RaBiT (Ours) & \textbf{67.56}  & 56.71  & 69.06  & 39.76  & 77.42  & 62.10 \\
    \midrule
    \multirow{4}[2]{*}{Llama3-8B} & FullPrecision & 72.93  & 60.08  & 80.30  & 50.17  & 79.76  & 67.80  \\
          & QTIP & \textbf{70.24}  & \textbf{55.53}  & 75.29  & 41.64  & 76.71  & 63.88  \\
          & DBF & 68.90  & 54.49  & 74.62  & 39.76  & 76.44  & 62.84  \\
          & RaBiT (Ours) & 69.37  & 55.13  & \textbf{75.37}  & \textbf{42.83}  & \textbf{77.96}  & \textbf{64.13}  \\
    \bottomrule
\end{tabular}%
    }
\label{tab:zero_shot_main_detail_appendix}
\end{table*}

\begin{table*}[ht]
  \centering
  \caption{\textbf{Detailed Zero-Shot Reasoning Accuracy on Gemma Models} (\%). Comparison of FP16 against leading 2-bit methods on five common benchmarks.}
\resizebox{0.90\textwidth}{!}{
  \setlength{\tabcolsep}{5.5pt}
    \begin{tabular}{llcccccc}
    \toprule
    \textbf{Models} & \textbf{Method} & \textbf{WinoGrande$\uparrow$} & \textbf{HellaSwag$\uparrow$} & \textbf{ARC-e$\uparrow$} & \textbf{ARC-c$\uparrow$}  & \textbf{PIQA$\uparrow$}   & \textbf{Average$\uparrow$} \\
    \midrule
    \multirow{4}{*}{Gemma3-1B} 
        & FullPrecision & 59.59 & 47.30 & 72.22 & 35.32 & 74.65 & 57.82 \\
        & QTIP          & 54.62 & 38.24 & 63.93 & 25.85 & 68.88 & 50.30 \\
        & DBF           & \textbf{58.01} & 40.37 & 62.92 & 28.41 & 70.18 & 51.98 \\
        & RaBiT (Ours)  & 56.59 & \textbf{42.94} & \textbf{64.52} & \textbf{29.44} & \textbf{72.42} & \textbf{53.18} \\
    \midrule
    \multirow{4}{*}{Gemma3-4B} 
        & FullPrecision & 69.22 & 56.77 & 81.52 & 51.45 & 79.05 & 67.60 \\
        & QTIP          & \textbf{66.85} & 52.25 & \textbf{77.53} & \textbf{44.62} & 76.12 & \textbf{63.47} \\
        & DBF           & 63.69 & 50.15 & 74.74 & 40.87 & 75.08 & 60.91 \\
        & RaBiT (Ours)  & 65.19 & \textbf{52.57} & 75.04 & 41.38 & \textbf{76.88} & 62.21 \\
    \midrule
    \multirow{4}{*}{Gemma3-12B} 
        & FullPrecision & 75.45 & 61.98 & 87.08 & 61.60 & 81.12 & 73.45 \\
        & QTIP          & \textbf{72.69} & 57.99 & \textbf{84.09} & \textbf{54.95} & 78.73 & \textbf{69.69} \\
        & DBF           & 72.14 & 57.20 & 82.49 & 52.05 & 77.97 & 68.37 \\
        & RaBiT (Ours)  & 72.30 & \textbf{58.45} & 82.41 & 52.13 & \textbf{78.95} & 68.85 \\
    \bottomrule
\end{tabular}

    }
\label{tab:zero_shot_gemma3_detail_appendix}
\end{table*}

\section{Inference Performance Analysis}
\label{appendix:inference_performance_analysis}
\subsection{Kernel Design}
\label{appendix:kernel_design}
Our CUDA kernels implement binary GEMV operations tailored to the memory-bound regime typical of the decoding phase in LLM inference. The design centers on bit-packing to reduce global memory traffic, with a latency-tolerant and matmul-free compute pipeline that leverages register-level staging.

\paragraph{Weight Packing.} To reduce memory traffic, each group of 32 columns is mapped to a \texttt{uint32\_t}, with $+1 \mapsto 0$ and $-1 \mapsto 1$. We then group the 32-bit words into \texttt{uint2} or \texttt{PackedBits3} (3 \texttt{uint32\_t} weights with padding), for 2-bit (2 binary weights) and 3-bit (3 binary weights) models, respectively. Rows are interleaved into warp-sized groups, ensuring that a warp issues full coalesced memory transactions when loading weights. Our efficient packing reduces the raw footprint of weights by a factor of $32\times$, compared to full-precision weights.

\paragraph{Compute Pipeline.} Each warp is assigned a set of output rows to avoid inter-warp synchronization. Input activations ($\mathbf{x}$) and column scales ($\mathbf{g}$) are read as vectorized \texttt{uint4} chunks. Binary signs are applied via lane-local bit shifts and XOR masks, instead of matrix multiplication. The kernel uses simple yet effective pipelining: while one tile of data is consumed, the subsequent tile is prefetched into registers. Accumulation proceeds using \texttt{half2} fused multiply-add intrinsics (\texttt{\_\_hfma2}), which increase arithmetic throughput without resorting to shared memory. Finally, reductions across threads in a warp are performed with shuffle operations, and output scale factors ($\mathbf{h}$) are applied in \texttt{fp16} precision. Notably, our architecture enables per-path parallelizable computation - instead of an n-bit weight, we parallelize with n 1-bit operations.

Efficient weight packing and pipelining reduce global memory access and raise the utilization of execution units on the GPU. The kernel therefore shifts the limiting factor from raw memory bandwidth toward register throughput, yielding measurable efficiency gains during the decoding stage of LLM inference, showing remarkable performance. 

\begin{figure}[b!]
    \centering
    \includegraphics[width=0.8\linewidth]{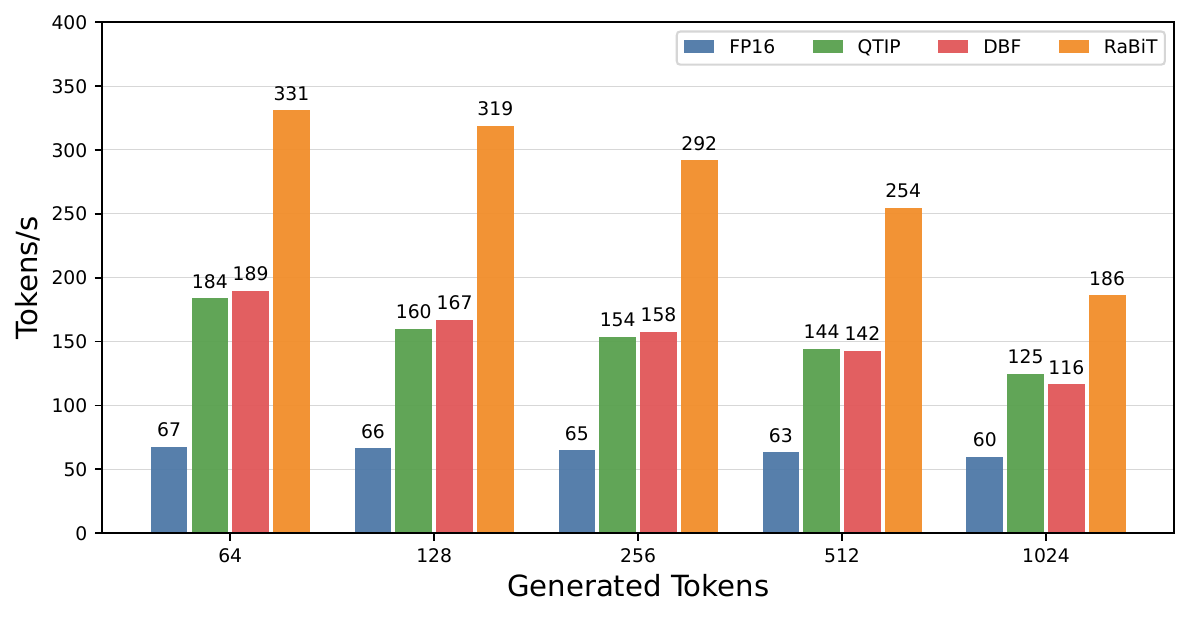}
    \caption{\textbf{End-to-End Decoding Throughput (Tokens/Second) for Llama2-7B on an NVIDIA RTX 4090 Across Various Generated Token Lengths.} RaBiT's parallel architecture consistently delivers superior performance over other 2-bit methods.}
    \label{fig:llama2_7b_decode_bar}
\end{figure}

\subsection{More Comparisons}
\label{appendix:more_comparisons}
To provide a more detailed analysis of the end-to-end inference speed, we benchmarked RaBiT against several key baselines: the full-precision (FP16) model, QTIP as the state-of-the-art Vector Quantization (VQ) method, and DBF, which features a similar stacked binary architecture. For QTIP, we utilized the publicly available CUDA kernels from the official implementation\footnote{\url{https://github.com/Cornell-RelaxML/qtip}}. For DBF, which also uses a stacked binary design but executes its two paths sequentially, we developed an optimized CUDA kernel that runs approximately 21\% faster than their public Triton-based implementation to ensure a fair and robust comparison. All evaluations were conducted on an NVIDIA RTX 4090 with \texttt{torch.compile} using the Llama2-7B model.

The results, depicted in \Cref{fig:llama2_7b_decode_bar}, were benchmarked across a range of generated token lengths (64, 128, 256, 512, and 1024) for a comprehensive analysis. As expected, all 2-bit methods significantly outperform the FP16 baseline due to the 8$\times$ reduction in memory bandwidth requirements. More importantly, \textbf{RaBiT demonstrates a substantial performance advantage, achieving nearly twice the decoding throughput} of the other 2-bit quantization methods. This speed-up stems directly from the efficiency of our parallel, matmul-free architecture. Unlike DBF, which is bottlenecked by its sequential computation of two binary paths, RaBiT's fully parallel design allows it to maximally leverage the benefits of its efficient binary cores. While the absolute tokens/second rate naturally decreases with longer generation sequences, the relative performance gap between the methods remains consistent, confirming the robustness of RaBiT's architectural advantage.

\section{Hyperparameters}
\label{appendix:hyper_parameters}
\subsection{Training Details}
We detail the hyperparameters used for our Quantization-aware training (QAT) experiments in \Cref{tab:training_detail}. All models were trained for 6 epochs using the Muon optimizer~\citep{jordan2024muon} with a cosine learning rate decay schedule. The models were initialized using our proposed function-aware strategy, with a fixed SVID iteration count of $T_{\max}=20$. Key hyperparameters, such as the learning rate and the I/O Channel Importance Scaling intensities ($\alpha_{\mathrm{in}}, \alpha_{\mathrm{out}}$), were fine-tuned for each specific model to achieve the best performance. All experiments were conducted on a single node equipped with four NVIDIA H100 GPUs.

\begin{table}[ht]
  \caption{\textbf{Hyperparameter Configuration.} We detail the training settings including learning rates, intensities $(\alpha_{in}, \alpha_{out})$, and batch information for the Llama and Gemma model families.}
  \label{tab:training_detail}
  \centering
  \resizebox{0.75\textwidth}{!}{
    \begin{tabular}{@{}llcccccc@{}}
    \toprule
    \multicolumn{2}{c}{Training Setup} & \multicolumn{2}{c}{Llama2} & \multicolumn{1}{c}{Llama3} & \multicolumn{3}{c}{Gemma3} \\[4pt]
    \cmidrule(lr){3-4} \cmidrule(lr){5-5} \cmidrule(lr){6-8}
    Bit & Target & 7B & 13B & 8B & 1B & 4B & 12B \\ 
    \midrule
    \multirow{2}{*}{2} 
      & Intensities ($\alpha_{\mathrm{in}}, \alpha_{\mathrm{out}}$) & (0.8, 0.65) & (0.95, 0.45) & (0.85, 0.7) & (0.85, 0.7) & (0.95, 0.7) & (0.75, 0.6) \\
      & Iteration ($T_{\max}$) & 20 & 20 & 20 & 20 & 20 & 20 \\
      & Learning Rate & 12e-6 & 1e-5 & 1e-5 & 1e-5 & 1e-5 & 5e-6 \\
      & Epoch & 6 & 6 & 6 & 6 & 6 & 6 \\
      & \# GPUs & 1 $\times$ 4 & 1 $\times$ 4 & 1 $\times$ 4 & 1 $\times$ 4 & 1 $\times$ 4 & 1 $\times$ 4  \\
      & \# Training Hours & 39 & 56 & 38 & 8 & 23 & 67 \\
    \midrule
    \multirow{3}{*}{3} 
      & Intensities ($\alpha_{\mathrm{in}}, \alpha_{\mathrm{out}}$) & (0.8, 0.65) & (0.95, 0.45) & (0.85, 0.7) & - & - & - \\
      & Iteration ($T_{\max}$) & 20 & 20 & 20 & - & - & - \\
      & Learning Rate & 1e-5 & 1e-5 & 1e-5 & - & - & - \\
      & Epoch & 6 & 6 & 6 & - & - & - \\
      & \# GPUs & 1 $\times$ 4 & 1 $\times$ 4 & 1 $\times$ 4 & - & - & -  \\
      & \# Training Hours & 46 & 88 & 44 & - & - & - \\
    \bottomrule
\end{tabular}%
  }
\end{table}

\begin{table}[h]
    \centering
    \caption{\textbf{Training Wall-Clock Time Overhead (Llama2-7B, 4×H100).} RaBiT incurs a modest overhead over Single-Path INT2 QAT while being substantially faster and more memory-efficient than Standard Independent QAT.}
    \label{tab:training_overhead}
    \vspace{-5pt}
    \small
    \resizebox{0.85\textwidth}{!}{
    \begin{tabular}{llrr}
        \toprule
        \textbf{Method} & \textbf{Architecture} & \textbf{Wall-Clock Time / Epoch} & \textbf{Overhead} \\
        \midrule
        Single-Path INT2 QAT & Non-Residual & 5.67 h & Baseline \\
        Standard Independent QAT & Residual (2 latent weights) & 10.33 h & +82.2\% \\
        \textbf{RaBiT (Ours)} & Residual (1 shared weight) & \textbf{6.50 h} & \textbf{+14.7\%} \\
        \bottomrule
    \end{tabular}
    }
\end{table}

\subsection{Grid Search for I/O Channel Importance Scaling Intensities}
\label{appendix:io_scale}
To determine the optimal intensity hyperparameters for our I/O Channel Importance Scaling (\Cref{subsec:stable_init}), we performed a comprehensive grid search. The objective was to identify the values of $\alpha_{\mathrm{in}}$ and $\alpha_{\mathrm{out}}$ that minimized the initial Knowledge Distillation (KD) loss post-initialization. This process utilized a calibration dataset of 128 samples randomly selected from the training data to measure the loss.

The example results of this search on the Llama2-7B model are detailed in \Cref{tab:io_scale_grid_search}. We observed a clear optimum, with the minimum initial KL divergence loss of 2,672 achieved at the configuration of $\alpha_{\mathrm{in}}=0.80$ and $\alpha_{\mathrm{out}}=0.65$. This finding underscores the importance of a balanced preconditioning strategy that considers both input activation statistics and output gradient magnitudes. \textbf{We repeated this grid search process for all other models to find their optimal alpha values.}

\begin{table}[ht]
    \centering
    \caption{\textbf{Grid Search Results for I/O Channel Importance Scaling Intensities ($\alpha_{\mathrm{in}}$, $\alpha_{\mathrm{out}}$) on Llama2-7B.} The metric is the Initial KL Divergence Loss (Lower is better). The optimal configuration is highlighted in bold.}
    \label{tab:io_scale_grid_search}
    \resizebox{0.30\textwidth}{!}{
    \begin{tabular}{c|cccc}
        \toprule
        \multirow{2}{*}{$\alpha_{\mathrm{out}}$} & \multicolumn{4}{c}{$\alpha_{\mathrm{in}}$} \\
        & 0.75 & \textbf{0.80} & 0.85 & 0.90  \\
        \midrule
        0.55 & 3,100 & 2,932 & 2,984 & 3,108  \\
        0.60 & 2,932 & 2,938 & 2,971 & 3,143  \\
        \textbf{0.65} & 3,167 & \textbf{2,672} & 2,697 & 3,063  \\
        0.70 & 3,083 & 2,821 & 2,983 & 3,462  \\
        \bottomrule
\end{tabular}

    }
\end{table}

\subsection{SVID Iteration Convergence Analysis}
\label{appendix:svid_iters}
The Iterative Residual SVID initialization (\Cref{subsec:stable_init}) aims to mitigate the scheduling bias inherent in standard greedy initialization. We analyzed the required number of iterations ($T_{\max}$) for convergence on the Llama2-7B model. We measured the Initial KL divergence loss as the iterations progressed from 1 (equivalent to Greedy SVID) up to 35.

The results, shown in \Cref{fig:kl_loss_svid_iter}, indicate that the initialization quality improves rapidly in the initial phase. The loss stabilizes significantly around 15 iterations, and the optimum is reached at 20 iterations. Beyond this point, further iterations do not provide additional benefits. Based on this analysis, we selected $T_{\max}=20$ as the default setting for RaBiT initialization, providing a robust convergence margin while maintaining an optimal balance between initialization quality and computational cost.

\begin{figure}[ht]
    \centering
    \includegraphics[width=0.47\linewidth]{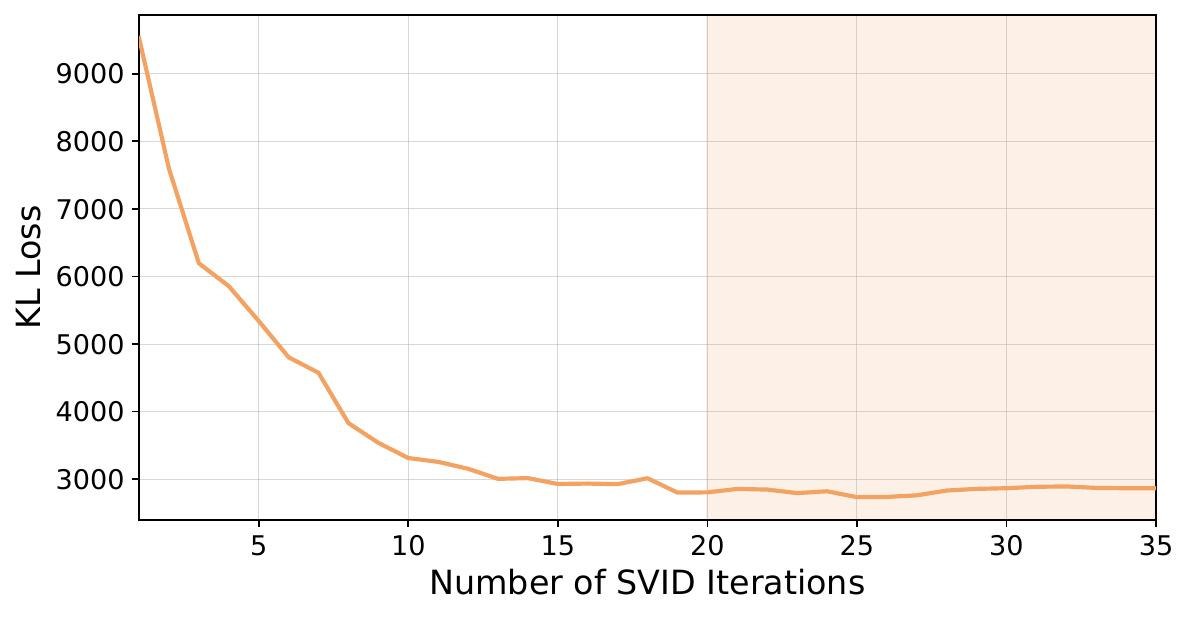}
    \caption{\textbf{Convergence Analysis of Iterative Residual SVID on Llama2-7B}. The metric is the Initial KL Divergence Loss (Lower is better). Convergence stabilizes around 20 iterations.}
    \label{fig:kl_loss_svid_iter}
\end{figure}

\subsection{Robustness to Random Seed and Optimizers}
\label{appendix:robustness_efficiency}

To ensure the reliability of RaBiT, we provide additional analyses on training stability across random seeds and  optimizer sensitivity (AdamW vs. Muon).

\paragraph{Multi-Seed Stability.}
We conducted experiments on Llama2-7B using 5 different random seeds to verify reproducibility. As shown in \Cref{tab:seed_stability}, RaBiT exhibits exceptional stability with a negligible standard deviation in both Perplexity (PPL) and Zero-shot QA accuracy, confirming that our results are robust and not artifacts of a specific seed.

\begin{table}[h]
    \centering
    \caption{\textbf{Multi-Seed Stability Analysis (Llama2-7B, 2-bit).} RaBiT demonstrates consistent performance across five random seeds.}
    \label{tab:seed_stability}
    \vspace{-5pt}
    \small
    \resizebox{0.75\textwidth}{!}{
    \begin{tabular}{lcccccc}
        \toprule
        \textbf{Metric} & \textbf{Seed 123} & \textbf{Seed 1019} & \textbf{Seed 1024} & \textbf{Seed 1112} & \textbf{Seed 1204} & \textbf{Avg. (Std.)} \\
        \midrule
        WikiText-2 PPL ($\downarrow$) & 5.78 & 5.77 & 5.77 & 5.78 & 5.78 & \textbf{5.78} (0.005) \\
        C4 PPL ($\downarrow$)    & 7.63 & 7.62 & 7.63 & 7.63 & 7.64 & \textbf{7.63} (0.006) \\
        QA Avg. ($\uparrow$)     & 61.65 & 61.60 & 61.36 & 61.90 & 61.34 & \textbf{61.57} (0.23) \\
        \bottomrule
    \end{tabular}
    }
\end{table}

\paragraph{Optimizer Sensitivity (AdamW vs. Muon).}
While our main experiments utilize the Muon optimizer for memory efficiency, we validated RaBiT with the standard AdamW optimizer to rule out optimizer dependency. On Llama2-7B, RaBiT trained with \textbf{AdamW achieved a PPL of 5.78}, which is statistically identical to the \textbf{5.78 PPL} obtained with Muon. This confirms that the performance gains stem from the coupled residual architecture, not the choice of optimizer.

\section{Extended Analysis of Inter-Path Adaptation}
\label{appendix:inter-path adaptation}

To provide a more granular view of the training dynamics, we conduct a layer-wise analysis of the Mean Squared Error (MSE) decomposition for the Llama2-7B model, visualized in \Cref{fig:inter_path_corr_layerwise}. This analysis offers two key insights into RaBiT's structural advantages over Standard QAT.

First, the results empirically confirm our central hypothesis across the network's depth. For most layers, RaBiT consistently generates a substantial negative covariance (the red-dashed component), which acts as a significant loss-reducing bonus, thereby lowering the total MSE. In contrast, Standard QAT fails to establish this effective error-cancellation, exhibiting a much smaller covariance term that provides negligible benefit. This provides strong visual evidence that RaBiT's coupled training successfully enforces the intended error-correction hierarchy, while Standard QAT suffers from the performance degradation of inter-path adaptation.

Second, and more strikingly, the analysis reveals RaBiT's ability to overcome a critical optimization challenge in extreme quantization: layer sensitivity. The Standard QAT baseline exhibits an exceptionally high MSE in the initial layers, a phenomenon consistent with the known sensitivity of early network layers to input distributions and quantization errors, as also observed in other LLMs by \citep{zhang2025towards, zhangtwo2024investigating}. RaBiT, however, dramatically suppresses this MSE peak. This suggests its benefits extend beyond merely enforcing anti-correlation. The fact that RaBiT tames this instability indicates that our method may resolve a more fundamental bottleneck in MSE-based QAT that has historically hindered extreme quantization in conventional architectures. While we designed RaBiT to foster negative correlation, its success in stabilizing these sensitive layers points to a deeper robustness. A full investigation into how residual coupling imparts this stability is a compelling direction for future research.

\begin{figure}[H]
\centering
\includegraphics[width=\textwidth]{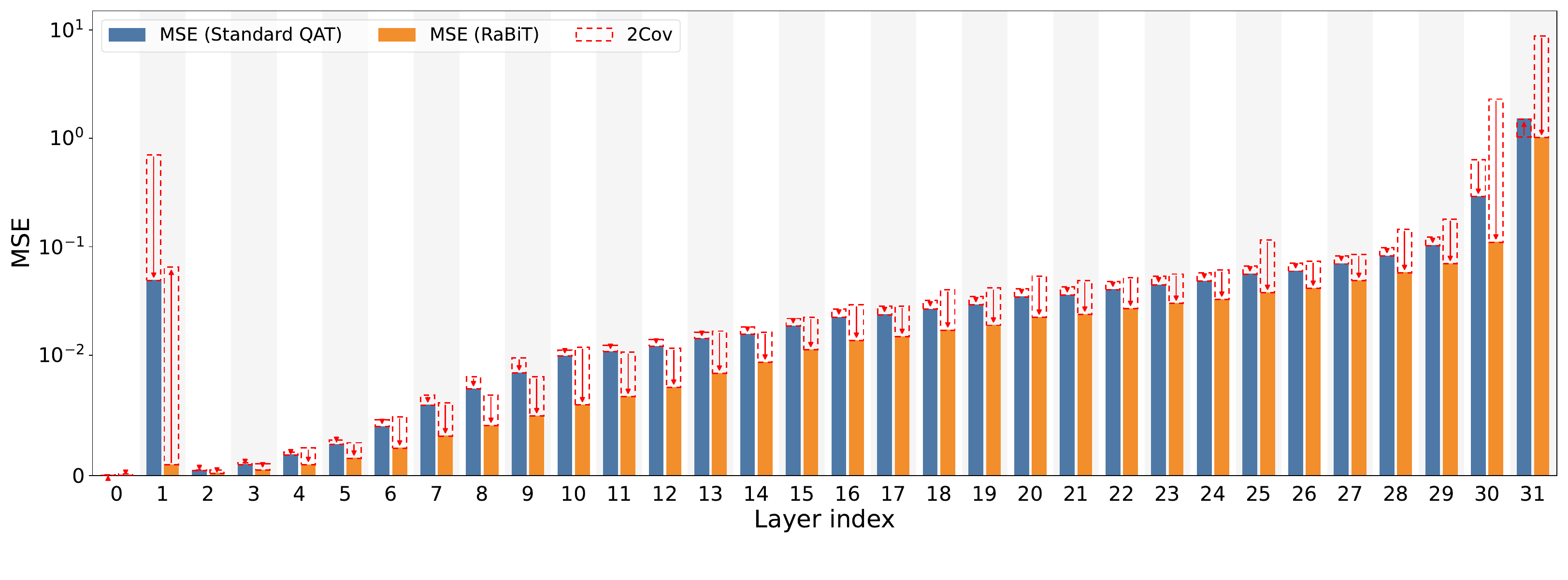}
\caption{
    \textbf{Layer-Wise MSE Decomposition in Llama2-7B's \texttt{down\_proj} Layers.}
    The bars compare the total Mean Squared Error (MSE) for Standard QAT (blue) and RaBiT (orange). The top of each bar represents the base error term ($C'$), while the red-dashed component visualizes twice the covariance ($2 \times \text{Cov}$). RaBiT consistently generates a large negative covariance, which actively reduces the total MSE, demonstrating effective error cancellation. Notably, RaBiT also suppresses the extremely high MSE peak observed in the early layers of the Standard QAT baseline, indicating its robustness against layer sensitivity.
}

\label{fig:inter_path_corr_layerwise}
\end{figure}

\newpage

\section{Generated Samples}

\subsection*{Sample 1: Newton}
\begin{flushleft}
\noindent\textbf{Prompt:} Newton discovered the law of universal gravitation. Universal gravitation is
\end{flushleft}

\begin{xltabular}{\linewidth}{@{}X@{}}
\toprule
Generated Output \\
\midrule
\endfirsthead

\multicolumn{1}{c}{\emph{Sample 1: Newton -- Continued}} \\
\toprule
Generated Output (Llama2-7b Rabitized 2bit) \\
\midrule
\endhead

\midrule
\multicolumn{1}{r}{\emph{Continued on next page}} \\
\endfoot
\bottomrule
\endlastfoot

\underline{Newton discovered the law of universal gravitation. Universal gravitation is} the force of gravity which draws objects towards the earth and other massive bodies. Gravity is the universal attraction between two objects and is proportional to the product of the masses of the two objects and inversely proportional to the square of the distance between them. Newton explained the force of gravity as being caused by the presence of matter in the universe. He stated that all matter attracted all other matter and that the tendency to move towards a massive body was proportional to the product of the masses of the objects. Newton developed the law of universal gravitation so that the attraction between two bodies could be calculated.  \\
\end{xltabular}
\medskip

\subsection*{Sample 2: Adam Smith}
\begin{flushleft}
\noindent\textbf{Prompt:} Adam Smith wrote The Wealth of Nations. The Wealth of Nations is
\end{flushleft}

\begin{xltabular}{\linewidth}{@{}X@{}}
\toprule
Generated Output \\
\midrule
\endfirsthead

\multicolumn{1}{c}{\emph{Sample 2: Adam Smith -- Continued}} \\
\toprule
Generated Output (Llama2-7b Rabitized 2bit) \\
\midrule
\endhead

\midrule
\multicolumn{1}{r}{\emph{Continued on next page}} \\
\endfoot

\bottomrule
\endlastfoot

\underline{Adam Smith wrote The Wealth of Nations. The Wealth of Nations is} the first book to outline the modern principles of capitalism and to make the case for it. Smith’s work is widely considered to be the first comprehensive system of political economy and one of the most important works of political philosophy. The Wealth of Nations is the foundation of modern capitalism. The Wealth of Nations is a book by Adam Smith, first published in 1776, which describes the nature of economic systems. Adam Smith is regarded as the father of modern economics. The Wealth of Nations is widely considered to be the first comprehensive system of political economy and one of the most important works of political philosophy. It is also one of the most influential books on economics. Adam Smith is regarded as the father of modern economics. \\
\end{xltabular}
\medskip

\subsection*{Sample 3: DNA double helix}
\begin{flushleft}
\noindent\textbf{Prompt:} Watson and Crick discovered the DNA double helix. The DNA double helix is
\end{flushleft}

\begin{xltabular}{\linewidth}{@{}X@{}}
\toprule
Generated Output \\
\midrule
\endfirsthead

\multicolumn{1}{c}{\emph{Sample 3: DNA double helix -- Continued}} \\
\toprule
Generated Output (Llama2-7b Rabitized 2bit) \\
\midrule
\endhead

\midrule
\multicolumn{1}{r}{\emph{Continued on next page}} \\
\endfoot

\bottomrule
\endlastfoot

\underline{Watson and Crick discovered the DNA double helix. The DNA double helix is} a structure found in the nucleus of every living cell. The double helix is the most fundamental structure in DNA. The structure is made of two single strands of DNA. The double helix was discovered in 1953 by James Watson and Francis Crick. In 1953, James Watson and Francis Crick discovered the double helix structure of DNA. The double helix structure is the most fundamental structure in DNA. The double helix structure is made of two single strands of DNA. The double helix is found in the nucleus of every living cell. \\
\end{xltabular}
\medskip

\newpage
\section{Algorithms}
% ===================== ALGORITHM 1 =====================
\begin{algorithm}[H]
    \caption{RaBiT Initialization: Calibrated Iterative Residual SVID}
    \label{alg:initialization}
    \centering
    \resizebox{0.95\linewidth}{!}{%
        \begin{minipage}{\linewidth}
            \begin{algorithmic}[1]
\STATE \textbf{Require:} Pretrained weight $\mat{W}_{\mathrm{FP}}$, Number of paths $k$, Max iterations $T_{\max}$
\STATE \textbf{Require:} Calibration stats $\vec{s}_{\mathrm{in}}, \vec{s}_{\mathrm{out}}$ and intensities $\alpha_{\mathrm{in}}, \alpha_{\mathrm{out}}$
\STATE \textbf{Output:} Initialized scales $\{(\vec{g}_i, \vec{h}_i)\}_{i=1}^k$
\medskip
\STATE \textcolor{gray!70!black}{\textit{// Step 1: I/O Channel Importance-Calibrated Preconditioning}}
\STATE Normalize: $\vec{s}_{\mathrm{in}} \leftarrow \vec{s}_{\mathrm{in}} / \max(\vec{s}_{\mathrm{in}})$, \quad $\vec{s}_{\mathrm{out}} \leftarrow \vec{s}_{\mathrm{out}} / \max(\vec{s}_{\mathrm{out}})$
\STATE Precondition: $\mat{W}' \leftarrow \vec{s}_{\mathrm{out}}^{\alpha_{\mathrm{out}}} \odot \mat{W}_{\mathrm{FP}} \odot \vec{s}_{\mathrm{in}}^{\alpha_{\mathrm{in}}}$
\medskip
\STATE \textcolor{gray!70!black}{\textit{// Step 2: Iterative Residual SVID}}
\STATE Initialize $\hat{\mat{W}}_i'^{(0)} \leftarrow \mat{0}$ for $i=1, \dots, k$
\FOR{$t=1$ to $T_{\max}$}
    \FOR{$i=1$ to $k$}
        \STATE \textcolor{gray!70!black}{\textit{// Calculate target residual (Gauss-Seidel style update)}}
        \STATE $\mat{R}_i'^{(t)} \leftarrow \mat{W}' - \left( \sum\nolimits_{j<i} \hat{\mat{W}}_j'^{(t)} + \sum\nolimits_{j>i} \hat{\mat{W}}_j'^{(t-1)}\right)$
        \STATE \textcolor{gray!70!black}{\textit{// Apply SVID to find the best rank-1 approximation}}
        \STATE $(\mat{B}_i'^{(t)}, \vec{g}_i'^{(t)}, \vec{h}_i'^{(t)}) \leftarrow \SVID(\mat{R}_i'^{(t)})$
        \STATE $\hat{\mat{W}}_i'^{(t)} \leftarrow \vec{g}_i'^{(t)} \odot \mat{B}_i'^{(t)} \odot \vec{h}_i'^{(t)}$
    \ENDFOR
\ENDFOR
\medskip
\STATE \textcolor{gray!70!black}{\textit{// Step 3: Map scales back to the original weight domain}}
\FOR{$i=1$ to $k$}
    \STATE $\vec{g}_i \leftarrow \vec{s}_{\mathrm{out}}^{-\alpha_{\mathrm{out}}} \odot \vec{g}_i'^{(T_{\max})}$
    \STATE $\vec{h}_i \leftarrow \vec{s}_{\mathrm{in}}^{-\alpha_{\mathrm{in}}} \odot \vec{h}_i'^{(T_{\max})}$
\ENDFOR
\end{algorithmic}
        \end{minipage}%
    }
\end{algorithm}

\vspace{-10pt}

% ===================== ALGORITHM 2 =====================
\begin{algorithm}[H]
    \caption{RaBiT: Residual-Aware Binarization Training (One Step)}
    \label{alg:training}
    \centering
    \resizebox{0.95\linewidth}{!}{%
        \begin{minipage}{\linewidth}
            \begin{algorithmic}[1]
\STATE \textbf{Parameters:} Shared full-precision weight $\mat{W}_{\mathrm{FP}}$; Scales $\{(\vec{g}_i, \vec{h}_i)\}_{i=1}^k$.
\STATE \textbf{Input:} Minibatch Input $\mat{X}$, Targets $\mat{T}$.
\vspace{1mm}
\STATE \textcolor{gray!70!black}{\textit{// 1. Forward Pass: On-the-fly Residual Coupling}}
\STATE $\mat{R}_0 \leftarrow \mat{W}_{\mathrm{FP}}$. \quad \textcolor{gray!70!black}{\textit{// Initialize residual with the shared weight}}
\STATE $\hat{\mat{W}}^{(k)} \leftarrow \mat{0}$. \quad \textcolor{gray!70!black}{\textit{// Effective weight for the entire layer}}
\FOR{$i=1$ to $k$}
    \STATE \textcolor{gray!70!black}{\textit{// Sequentially derive the i-th binary path}}
    \STATE $\mat{B}_i \leftarrow \sign(\mat{R}_{i-1})$.
    \STATE $\hat{\mat{W}}_i \leftarrow \vec{g}_i \odot \mat{B}_i \odot \vec{h}_i$.
    \STATE $\hat{\mat{W}}^{(k)} \leftarrow \hat{\mat{W}}^{(k)} + \hat{\mat{W}}_i$.
    
    \STATE \textcolor{gray!70!black}{\textit{// Update residual for the next path}}
    \STATE $\mat{R}_{i} \leftarrow \mat{R}_{i-1} - \hat{\mat{W}}_i$.
\ENDFOR
\STATE $\mat{Y} \leftarrow \hat{\mat{W}}^{(k)} \mat{X}$. \quad \textcolor{gray!70!black}{\textit{// Compute layer output}}
\STATE Calculate Loss $\mathcal{L}(\mat{Y}, \mat{T})$.
\vspace{1mm}
\STATE \textcolor{gray!70!black}{\textit{// 2. Backward Pass}}
\STATE $\Delta \leftarrow \partial\mathcal{L} / \partial \mat{Y}$. \quad\textcolor{gray!70!black}{\textit{(Output gradient)}}

\STATE \textcolor{gray!70!black}{\textit{// Surrogate gradient for the shared weight $\mat{W}_{\mathrm{FP}}$}}
\STATE $\nabla_{\mat{W}_{\mathrm{FP}}} \leftarrow \Delta \mat{X^\top}$.
\STATE \textcolor{gray!70!black}{\textit{// Gradients for scales (treating $\mat{B}_i$ as constant)}}
\FOR{$i=1$ to $k$}
    \STATE Compute $\nabla_{\vec{g}_i}$ and $\nabla_{\vec{h}_i}$ using $\Delta, \mat{B}_i, \mat{X}$, and other scales.
\ENDFOR
\vspace{1mm}
\STATE \textcolor{gray!70!black}{\textit{// 3. Parameter Update}}
\STATE Update $\{\mat{W}_{\mathrm{FP}}, (\vec{g}_i, \vec{h}_i)_{i=1}^k\}$ using an optimizer with the computed gradients.
\end{algorithmic}

        \end{minipage}%
    }
\end{algorithm}

\newpage
\section{Core Kernel Code}

\begingroup
\label{ker:rabit_kernel}
\begin{lstlisting}[
    style=cudaStyle,
    inputencoding=utf8,
    breaklines=true, breakatwhitespace=true, columns=fullflexible, keepspaces=true,
    numbers=left, numberstyle=\tiny, numbersep=8pt,
    xleftmargin=3em, framexleftmargin=3em  
  ]
#include <ATen/cuda/CUDAContext.h>
#include <c10/cuda/CUDAException.h>
#include <c10/cuda/CUDAGuard.h>
#include <c10/macros/Macros.h>
#include <cuda_fp16.h>
#include <cuda_runtime.h>
#include <torch/extension.h>

#include <sstream>
#include <string>
#include <type_traits>

// ========================================================================
// === Macros & Constants ===
// ========================================================================

#define CHECK_CUDA(x) TORCH_CHECK(x.is_cuda(), #x " must be a CUDA tensor")
#define CHECK_CONTIGUOUS(x) TORCH_CHECK(x.is_contiguous(), #x " must be contiguous")
#define CHECK_F16(x) TORCH_CHECK(x.scalar_type() == torch::kFloat16, #x " must be float16 tensor")
#define CHECK_INT32(x) TORCH_CHECK(x.scalar_type() == torch::kInt32, #x " tensor dtype must be int32")

#define CHECK_CUDA_CONT_F16(x) \
    do {                       \
        CHECK_CUDA(x);         \
        CHECK_CONTIGUOUS(x);   \
        CHECK_F16(x);          \
    } while (0)

#define CHECK_CUDA_CONT_INT32(x) \
    do {                         \
        CHECK_CUDA(x);           \
        CHECK_CONTIGUOUS(x);     \
        CHECK_INT32(x);          \
    } while (0)

constexpr int WARP_SIZE = 32;
constexpr uint32_t SIGN_MASK_U32 = 0x80008000u;
constexpr unsigned FULL_MASK = 0xffffffff;

// ========================================================================
// === Device Helper Functions ===
// ========================================================================

static __device__ __forceinline__ uint32_t half2_to_uint32(const half2& h) {
    return reinterpret_cast<const uint32_t&>(h);
}

static __device__ __forceinline__ half2 uint32_to_half2(uint32_t v) {
    return reinterpret_cast<const half2&>(v);
}

// Sign-flipping utility.
// Logic: bits determine if we flip the sign of the float16 values packed in uint4.
__device__ __forceinline__ uint4 apply_sign(uint4 x4, uint32_t bits, int shift_base) {
    uint32_t shifted = bits << shift_base;
    x4.x ^= shifted & SIGN_MASK_U32;
    x4.y ^= (shifted << 1) & SIGN_MASK_U32;
    x4.z ^= (shifted << 2) & SIGN_MASK_U32;
    x4.w ^= (shifted << 3) & SIGN_MASK_U32;
    return x4;
}

// ========================================================================
// === Kernel Implementation ===
// ========================================================================

template <unsigned NUM_THREAD, unsigned NUM_ROW_PER_WARP = 2, unsigned NUM_ACC = 4, unsigned NUM_PIPELINE_STAGE = 2>
__forceinline__ __device__ void rabit_2bit_half_impl(
    const half* __restrict__ x,
    const half* __restrict__ scale_g_0,
    const uint32_t* __restrict__ Wbits,
    const half* __restrict__ scale_h_0,
    const half* __restrict__ scale_g_1,
    const half* __restrict__ scale_h_1,
    half* __restrict__ y,
    int N, int M) {
    
    // Constraints
    static_assert(NUM_ROW_PER_WARP == 2, "Only 2 rows per warp supported currently");
    static_assert(NUM_THREAD % WARP_SIZE == 0, "Threads must be multiple of warp size");

    constexpr unsigned NUM_WARP = NUM_THREAD / WARP_SIZE;
    constexpr unsigned NUM_ROW_PER_BLOCK = NUM_ROW_PER_WARP * NUM_WARP;

    // Identifiers
    const int tid = threadIdx.x % WARP_SIZE;
    const int wid = threadIdx.x / WARP_SIZE;

    // Vectorized pointers (loading 128-bit chunks / 8 halves at a time)
    const uint4* x8s = reinterpret_cast<const uint4*>(x);
    const uint4* sg8s_0 = reinterpret_cast<const uint4*>(scale_g_0);
    const uint4* sg8s_1 = reinterpret_cast<const uint4*>(scale_g_1);

    // Data Structures for Pipelining
    struct RowAccum {
        half2 plane[2][NUM_ACC];
    };
    struct ColumnScalePair {
        uint4 g0;
        uint4 g1;
    };
    struct StageTile {
        uint4 x;
        ColumnScalePair scale;
        uint32_t bits;
    };
    struct RowContext {
        int index;
        RowAccum acc;
    };

    RowContext rows[NUM_ROW_PER_WARP] = {};
    StageTile pipe[NUM_PIPELINE_STAGE];

    // Setup Row Indices
    const int block_row_base = blockIdx.x * NUM_ROW_PER_BLOCK;
    #pragma unroll
    for (unsigned irow = 0; irow < NUM_ROW_PER_WARP; ++irow) {
        rows[irow].index = block_row_base + static_cast<int>(wid) * NUM_ROW_PER_WARP + static_cast<int>(irow);
    }

    // Weight Bit Indexing
    // N8 = Number of 8-half chunks (128-bit words)
    const int N8 = N >> 3;
    const int wordsN = N8; 
    const int tile_stride = wordsN;
    const int tile_idx = blockIdx.x * NUM_WARP + wid;
    const uint32_t* tile_ptr = Wbits + tile_idx * tile_stride;

    // --- Helper Lambdas ---

    auto load_to_reg = [&](int ireg, int idx8) {
        StageTile& stage = pipe[ireg];
        stage.x = x8s[idx8];
        stage.scale.g0 = sg8s_0[idx8];
        stage.scale.g1 = sg8s_1[idx8];
        stage.bits = *(tile_ptr + idx8);
    };

    // Accumulator indices for circular buffer usage or simple unrolling
    const int MASK = NUM_ACC - 1;
    auto fma = [&](const uint4& scale, const uint4& x, half2 acc[NUM_ACC]) {
        acc[0] = __hfma2(uint32_to_half2(scale.x), uint32_to_half2(x.x), acc[0]);
        acc[1 & MASK] = __hfma2(uint32_to_half2(scale.y), uint32_to_half2(x.y), acc[1 & MASK]);
        acc[2 & MASK] = __hfma2(uint32_to_half2(scale.z), uint32_to_half2(x.z), acc[2 & MASK]);
        acc[3 & MASK] = __hfma2(uint32_to_half2(scale.w), uint32_to_half2(x.w), acc[3 & MASK]);
    };

    auto calc_main = [&](int ireg) {
        StageTile& stage = pipe[ireg];
        #pragma unroll
        for (unsigned irow = 0; irow < NUM_ROW_PER_WARP; ++irow) {
            RowAccum& acc = rows[irow].acc;
            // Apply sign logic based on bits (0 or 1 selects path)
            // 0 + 8*irow and 4 + 8*irow are specific bit shifts for the packed format
            fma(stage.scale.g0, apply_sign(stage.x, stage.bits, 0 + 8 * irow), acc.plane[0]);
            fma(stage.scale.g1, apply_sign(stage.x, stage.bits, 4 + 8 * irow), acc.plane[1]);
        }
    };

    // --- Main Loop (Pipelined) ---

    int idx_load = tid;
    int idx_calc = tid;

    // Prologue: Fill pipeline
    #pragma unroll
    for (int istg = 0; istg < NUM_PIPELINE_STAGE; ++istg) {
        if (idx_load < N8) {
            load_to_reg(istg, idx_load);
        }
        idx_load += WARP_SIZE;
    }

    // Steady state
    const int bound_mainloop = N8 - (NUM_PIPELINE_STAGE * WARP_SIZE);
    while (idx_calc < bound_mainloop) {
        #pragma unroll
        for (int istg = 0; istg < NUM_PIPELINE_STAGE; ++istg) {
            calc_main(istg);
        }
        #pragma unroll
        for (int istg = 0; istg < NUM_PIPELINE_STAGE; ++istg) {
            if (idx_load < N8) {
                load_to_reg(istg, idx_load);
            }
            idx_load += WARP_SIZE;
            idx_calc += WARP_SIZE;
        }
    }

    // Epilogue: Drain pipeline
    #pragma unroll
    for (int istg = 0; istg < NUM_PIPELINE_STAGE; ++istg) {
        if (idx_calc < N8) {
            calc_main(istg);
            idx_calc += WARP_SIZE;
        }
        if (idx_load < N8) {
            load_to_reg(istg, idx_load);
            idx_load += WARP_SIZE;
        }
    }

    // --- Reduction & Output ---

    auto warp_sum_half = [](half v) {
        half2 h2 = __halves2half2(v, __float2half(0.0f));
        #pragma unroll
        for (int off = WARP_SIZE >> 1; off > 0; off >>= 1) {
            half2 other = __shfl_xor_sync(FULL_MASK, h2, off);
            h2 = __hadd2(h2, other);
        }
        return __low2half(h2);
    };

    #pragma unroll
    for (unsigned irow = 0; irow < NUM_ROW_PER_WARP; ++irow) {
        RowContext& row = rows[irow];
        const int out_i = row.index;
        RowAccum& acc = row.acc;

        half2 lane_acc_0 = __float2half2_rn(0.0f);
        half2 lane_acc_1 = __float2half2_rn(0.0f);

        #pragma unroll
        for (unsigned iacc = 0; iacc < NUM_ACC; ++iacc) {
            lane_acc_0 = __hadd2(lane_acc_0, acc.plane[0][iacc]);
            lane_acc_1 = __hadd2(lane_acc_1, acc.plane[1][iacc]);
        }

        // Horizontal sum within registers (half2 -> half)
        half sum0_h = __hadd(__low2half(lane_acc_0), __high2half(lane_acc_0));
        half sum1_h = __hadd(__low2half(lane_acc_1), __high2half(lane_acc_1));

        // Apply final scaling factors (h0, h1) and combine paths
        half lane_total_h = __hfma(sum0_h, scale_h_0[out_i], __hmul(sum1_h, scale_h_1[out_i]));

        // Cross-lane reduction
        half warp_sum_h = warp_sum_half(lane_total_h);

        if (tid == 0) {
            y[out_i] = warp_sum_h;
        }
    }
}

template <unsigned NUM_THREAD, unsigned NUM_ROW_PER_WARP = 2, unsigned NUM_ACC = 4, unsigned NUM_PIPELINE_STAGE = 2>
__launch_bounds__(NUM_THREAD) __global__ void rabit_2bit_half_dyn_kernel(
    const half* __restrict__ x,
    const half* __restrict__ scale_g_0,
    const uint32_t* __restrict__ Wbits,
    const half* __restrict__ scale_h_0,
    const half* __restrict__ scale_g_1,
    const half* __restrict__ scale_h_1,
    half* __restrict__ y,
    int N, int M) {
    
    const int s = blockIdx.y;
    // Offset for batch size (sequence length)
    const half* x_ptr_for_this_seq = x + s * N;
    half* y_ptr_for_this_seq = y + s * M;

    rabit_2bit_half_impl<NUM_THREAD, NUM_ROW_PER_WARP, NUM_ACC, NUM_PIPELINE_STAGE>(
        x_ptr_for_this_seq, scale_g_0, Wbits, scale_h_0, scale_g_1, scale_h_1, y_ptr_for_this_seq, N, M);
}

// ========================================================================
// === Template Dispatcher ===
// ========================================================================

// Helper to sweep through a compile-time list of values.
template <typename T, T... values>
struct parameter_sweep {
    template <typename F, typename U>
    void operator()(const char* name, U runtime_value, F&& func) const {
        bool found = false;
        // Fold expression to iterate over template values
        (
            [&]() {
                if (!found && runtime_value == values) {
                    std::forward<decltype(func)>(func)(std::integral_constant<T, values>{});
                    found = true;
                }
            }(),
            ...);

        if (!found) {
            std::stringstream ss;
            bool is_first = true;
            (((is_first ? ss : (ss << ", ")) << values, is_first = false), ...);
            TORCH_CHECK(false, name, " should be one of [", ss.str(), "], but ", runtime_value, " was given");
        }
    }
};

// ========================================================================
// === C++ Interface ===
// ========================================================================

torch::Tensor rabit_2bit_half_dyn_forward(
    torch::Tensor x,
    torch::Tensor scale_g_0, torch::Tensor Wbits, torch::Tensor scale_h_0,
    torch::Tensor scale_g_1, torch::Tensor scale_h_1,
    int64_t num_thread, int64_t num_row_per_warp, int64_t num_acc, int64_t num_pipeline_stage) {
    
    CHECK_CUDA_CONT_F16(x);
    CHECK_CUDA_CONT_F16(scale_g_0);
    CHECK_CUDA_CONT_INT32(Wbits);
    CHECK_CUDA_CONT_F16(scale_h_0);
    CHECK_CUDA_CONT_F16(scale_g_1);
    CHECK_CUDA_CONT_F16(scale_h_1);

    TORCH_CHECK(x.dim() == 1 || x.dim() == 2, "Dynamic Half: x must be 1D or 2D");

    const int seqlen = x.dim() == 1 ? 1 : static_cast<int>(x.size(0));
    const int N = static_cast<int>(x.size(-1));
    const int M = static_cast<int>(scale_h_0.size(0));

    // Validations
    TORCH_CHECK(scale_g_0.size(0) == N, "Dynamic Half: scale_g_0 size mismatch");
    TORCH_CHECK(scale_h_0.size(0) == M, "Dynamic Half: scale_h_0 size mismatch");
    TORCH_CHECK(scale_g_1.size(0) == N, "Dynamic Half: scale_g_1 size mismatch");
    TORCH_CHECK(scale_h_1.size(0) == M, "Dynamic Half: scale_h_1 size mismatch");
    TORCH_CHECK((N % 32 == 0), "Dynamic Half: N must be multiple of 32");
    TORCH_CHECK(num_row_per_warp > 0, "Dynamic Half: num_row_per_warp must be positive");

    const int rows_per_warp = static_cast<int>(num_row_per_warp);
    TORCH_CHECK(M % rows_per_warp == 0, "Dynamic Half: rows_per_warp must divide output size");
    TORCH_CHECK(Wbits.dim() == 2 && Wbits.size(0) == M / rows_per_warp && Wbits.size(1) == N / 8,
                "Dynamic Half: Wbits shape mismatch");

    auto y = torch::empty({(long long)seqlen, (long long)M}, x.options());

    at::cuda::CUDAGuard guard(x.device());
    auto stream = at::cuda::getCurrentCUDAStream();

    // Template Sweepers
    parameter_sweep<unsigned, 32u, 64u, 128u, 256u, 512u, 1024u> sweep_num_thread;
    parameter_sweep<unsigned, 2> sweep_row_per_warp;
    parameter_sweep<unsigned, 1, 2, 4> sweep_num_acc;
    parameter_sweep<unsigned, 1, 2, 3, 4> sweep_pipeline_stage;

    // Dispatcher
    auto kernel_func = [=](auto num_thread_c, auto num_row_per_warp_c, auto num_acc_c, auto num_pipeline_stage_c) {
        constexpr unsigned num_thread_v = num_thread_c.value;
        constexpr unsigned num_row_per_warp_v = num_row_per_warp_c.value;
        constexpr unsigned num_acc_v = num_acc_c.value;
        constexpr unsigned pipeline_stage_v = num_pipeline_stage_c.value;

        // Runtime check against compile time constant wrapper
        if (num_row_per_warp_v != static_cast<unsigned>(rows_per_warp)) return;

        constexpr unsigned NUM_ROW_PER_BLOCK = (num_thread_v / WARP_SIZE) * num_row_per_warp_v;
        TORCH_CHECK(M % NUM_ROW_PER_BLOCK == 0, "Output size ", M, " must be div by block rows ", NUM_ROW_PER_BLOCK);

        dim3 grid_dim(M / NUM_ROW_PER_BLOCK, seqlen);

        rabit_2bit_half_dyn_kernel<num_thread_v, num_row_per_warp_v, num_acc_v, pipeline_stage_v>
            <<<grid_dim, num_thread_v, 0, stream>>>(
                reinterpret_cast<const half*>(x.data_ptr<at::Half>()),
                reinterpret_cast<const half*>(scale_g_0.data_ptr<at::Half>()),
                reinterpret_cast<const uint32_t*>(Wbits.data_ptr<int32_t>()),
                reinterpret_cast<const half*>(scale_h_0.data_ptr<at::Half>()),
                reinterpret_cast<const half*>(scale_g_1.data_ptr<at::Half>()),
                reinterpret_cast<const half*>(scale_h_1.data_ptr<at::Half>()),
                reinterpret_cast<half*>(y.data_ptr<at::Half>()),
                N, M);
        C10_CUDA_KERNEL_LAUNCH_CHECK();
    };

    // Execute Sweep
    sweep_num_thread("num_thread", num_thread, [&](auto num_thread_c) {
        sweep_row_per_warp("num_row_per_warp", num_row_per_warp, [&](auto num_row_per_warp_c) {
            sweep_num_acc("num_acc", num_acc, [&](auto num_acc_c) {
                sweep_pipeline_stage("num_pipeline_stage", num_pipeline_stage, [&](auto num_pipeline_stage_c) {
                    kernel_func(num_thread_c, num_row_per_warp_c, num_acc_c, num_pipeline_stage_c);
                });
            });
        });
    });

    if (x.dim() == 1) {
        y.squeeze_(0);
    }
    return y;
}
\end{lstlisting}
\endgroup
%%%%%%%%%%%%%%%%%%%%%%%%%%%%%%%%%%%%%%%%%%%%%%%%%%%%%%%%%%%%%%%%%%%%%%%%%%%%%%%
%%%%%%%%%%%%%%%%%%%%%%%%%%%%%%%%%%%%%%%%%%%%%%%%%%%%%%%%%%%%%%%%%%%%%%%%%%%%%%%

\end{document}